\documentclass[sigconf]{acmart}
\renewcommand\footnotetextcopyrightpermission[1]{} 
\usepackage{amsmath}
\usepackage[utf8]{inputenc}
\usepackage[english]{babel}
\usepackage{amsthm}
\usepackage{amsfonts}
\usepackage{graphicx}  
\usepackage{algorithm}
\usepackage{algpseudocode}
\usepackage{subcaption}
\usepackage{multirow}
\usepackage{caption}
\theoremstyle{definition}
\newtheorem{definition}{Definition}
\newtheorem{prop}{Proposition}
\newtheorem{theorem}{Theorem}

\newcommand\independent{\protect\mathpalette{\protect\independenT}{\perp}}
\def\independenT#1#2{\mathrel{\rlap{$#1#2$}\mkern2mu{#1#2}}}

\begin{document}

\title{FairGAN: Fairness-aware Generative Adversarial Networks}

\author{Depeng Xu}
\affiliation{
	\institution{University of Arkansas}
}
\email{depengxu@uark.edu}
\author{Shuhan Yuan}
\affiliation{
	\institution{University of Arkansas}
}
\email{sy005@uark.edu}
\author{Lu Zhang}
\affiliation{
	\institution{University of Arkansas}
}
\email{lz006@uark.edu}
\author{Xintao Wu}
\affiliation{
	\institution{University of Arkansas}
}
\email{xintaowu@uark.edu}

\begin{abstract}
Fairness-aware learning is increasingly important in data mining. Discrimination prevention aims to prevent discrimination in the training data before it is used to conduct predictive analysis. In this paper, we focus on fair data generation that ensures the generated data is discrimination free. Inspired by generative adversarial networks (GAN), we present fairness-aware generative adversarial networks, called FairGAN, which are able to learn a generator producing fair data and also preserving good data utility. Compared with the naive fair data generation models, FairGAN further ensures the classifiers which are trained on generated data can achieve fair classification on real data. Experiments on a real dataset show the effectiveness of FairGAN.
\end{abstract}


\maketitle

\section{Introduction}

Discrimination refers to unjustified distinctions in decisions against individuals based on their membership in a certain group. 
Currently, many organizations or institutes adopt machine learning models trained on historical data to automatically make decisions, including hiring, lending and policing \cite{Joseph2016Fairness}.
However, many studies have shown that machine learning models have biased performance against the \textit{protected group} \cite{Beutel2017Data,Binns2017Fairness,Bolukbasi2016Man}. In principle, if a dataset has discrimination against the protected group, the predictive model simply trained on the dataset will incur discrimination. 

Many approaches aim to mitigate discrimination from historical datasets. A general requirement of modifying datasets is to preserve the data utility while removing the discrimination. Some methods mainly modify the labels of the dataset \cite{Zhang2017Causal,Kamiran2009Classifying}. Some methods also revise the attributes of data other than the label, such as the Preferential Sampling \cite{Kamiran2012Data} and the Disparate Impact Removal \cite{Feldman2015Certifying}. In addition, some methods prevent discrimination by learning fair representations from a dataset \cite{Edwards2015Censoring,Zemel2013Learning,Louizos2016Variational}.

In this work, instead of removing the discrimination from the existing dataset, we focus on generating fair data.
Generative adversarial networks (GAN) have demonstrated impressive performance on modeling the real data distribution and generating high quality synthetic data that are similar to real data \cite{Goodfellow2014Generative,Radford2015Unsupervised}. After generating high quality synthetic data, many approaches adopt the synthetic dataset to conduct predictive analysis instead of using the real data, especially when the real data is very limited \cite{Choi2017Generating}. However, due to high similarity between the real data and synthetic data, if the real data incur discrimination, the synthetic data can also incur discrimination. The following predictive analysis which is based on the synthetic data can be subject to discrimination. 

Throughout the paper, for ease of representation, we assume that there is only one protected attribute,  which is a binary attribute associated with the domain values of the unprotected group  and the protected group. We also assume there is one binary decision attribute associated with the domain values of the positive decision and the negative decision. Formally, let $\mathcal{D} = \{\mathcal{X}, \mathcal{Y}, \mathcal{S}\}$ be a historical dataset where $\mathcal{X} \in \mathbb{R}^n$ is the unprotected attributes, $\mathcal{Y} \in \{0,1\}$ is the decision, and $\mathcal{S} \in \{0,1\}$ is the protected attribute. We aim to generate a fair dataset $\hat{\mathcal{D}} = \{\hat{\mathcal{X}}, \hat{\mathcal{Y}}, \hat{\mathcal{S}}\}$. In principle, the generated fair data $\hat{\mathcal{D}}$ should meet following requirements: 1) \textbf{data utility} which indicates the generated data should preserve the general relationship between attributes and decision in the real data; 2) \textbf{data fairness} which indicates there is no discrimination in the generated data; 3) \textbf{classification utility} which indicates classifiers trained on the generated data should achieve high accuracy when deployed for decision prediction of future real data; 4) \textbf{classification fairness} which indicates classifiers trained on the generated data should not incur discrimination when predicting on real data. 

To achieve high data utility and classification utility, the generated data should be close to the real data, i.e., $\mathcal{D} \approx \hat{\mathcal{D}}$. To achieve data fairness and classification fairness, the generated data should be free from discrimination. In particular, there are two types of discrimination in the literature, disparate treatment and disparate impact. 
\textbf{Disparate treatment} indicates the discriminatory outcomes are due to explicitly using the protected attribute to make decisions. Hence, data fairness can be achieved if the generated decision has no correlation with the generated protected attribute. 
\textbf{Disparate impact} indicates the discriminatory outcomes are not explicitly from the protected attribute but from the proxy unprotected attributes. Theoretical analysis shows that the disparate impact is caused by the correlation between the unprotected attributes and the protected attribute \cite{Feldman2015Certifying}. Hence, to achieve classification fairness, a certifying framework proposed in \cite{Feldman2015Certifying} indicates that the dataset adopted to train classifiers should be free from the \textit{disparate impact}. 


It is nontrivial to achieve the four requirements simultaneously. For example, a naive approach keeps the generated unprotected attributes and decision close to the real data $\{\hat{\mathcal{X}}, \hat{\mathcal{Y}}\} = \{\mathcal{X}, \mathcal{Y}\}$ but randomly sets the protected attribute $\hat{\mathcal{S}}$. With this strategy, the disparate treatment in the dataset $\hat{\mathcal{D}} = \{\hat{\mathcal{X}}, \hat{\mathcal{Y}}, \hat{\mathcal{S}}\}$ is removed, so the data fairness can be achieved. However, the classification fairness on real data cannot be ensured since the classifier may preserve the potential disparate impact from unprotected attributes. This is because when $\hat{\mathcal{X}} = \mathcal{X}$, given the generated unprotected attributes $\hat{\mathcal{X}}$, their corresponding real protected attribute $\mathcal{S}$ is predictable. The disparate impact is not mitigated from the generated dataset.


We develop fairness-aware generative adversarial networks (FairGAN) for fair data generation. 
Besides generating synthetic samples that match the distribution of real data, we also aim to prevent discrimination in the generated dataset. In particular, the generated data should be free from disparate treatment and disparate impact. FairGAN consists of one generator and two discriminators. The generator generates fake samples $\{\hat{\mathcal{X}}, \hat{\mathcal{Y}}\}$ conditioned on the protected attribute $\mathcal{S}$. One discriminator aims to ensure the generated data $\{\hat{\mathcal{X}}, \hat{\mathcal{Y}}, \hat{\mathcal{S}}\}$ close to the real data $\{\mathcal{X}, \mathcal{Y}, \mathcal{S}\}$ while the other discriminator aims to ensure there are no correlation between $\hat{\mathcal{X}}$ and  $\hat{\mathcal{S}}$ and no correlation between $\hat{\mathcal{Y}}$ and $\hat{\mathcal{S}}$. Note that $\hat{\mathcal{S}} = \mathcal{S}$ since generator is conditioned on $\mathcal{S}$. Hence, unlike the naive method, which has $\hat{\mathcal{X}} = \mathcal{X}$ and randomly generates $ \hat{\mathcal{S}}$, FairGAN generates revised unprotected attributes $\hat{\mathcal{X}}$ and decision $\hat{\mathcal{Y}}$ given the protected attribute $\hat{\mathcal{S}}$ ($\mathcal{S}$) and achieves $\hat{\mathcal{X}} \independent \mathcal{S}$ and $\hat{\mathcal{Y}} \independent \mathcal{S}$. Therefore, the generated data can meet requirements of data fairness and classification fairness. The experimental results show that FairGAN can achieve fair data generation with good data utility and free from disparate impact, so the classifiers trained on the synthetic datasets can achieve fair classification on the real data with high accuracy.


\section{Related Work}

With the widely adoption of automated decision making systems, fairness-aware learning or anti-discrimination learning becomes an increasingly important task. In fairness-aware learning,  discrimination prevention aims to remove discrimination by modifying the biased data and/or the predictive algorithms built on the data. Many approaches have been proposed for constructing discrimination-free classifiers, which can be broadly classified into three categories: the pre-process approaches that modify the training data to remove discriminatory effect before conducting predictive analytics \cite{Calders2009Building,Dwork2011Fairness,Zhang2017Causal,Feldman2015Certifying}, the in-process approaches that enforce fairness to  classifiers by introducing constraints or regularization terms to the objective functions \cite{Kamishima2011FairnessAware,Zafar2017Fairness}, and the post-process approaches that directly change the predicted labels \cite{Kamiran2010Discrimination,Hardt2016Equality}. 

The pre-process approaches that modify the training data are widely studied. The fundamental assumption of the pre-process methods is that, once a classifier is trained on a discrimination-free dataset, the prediction made by the classifier will also be discrimination free \cite{Kamiran2009Classifying}. Research in \cite{Kamiran2012Data} proposed several methods for modifying data including Massaging, which corrects the labels of some individuals in the data, Reweighting, which assigns weights to individuals to balance the data, and Sampling, which changes the sample sizes of different subgroups to remove the discrimination in the data. In \cite{Feldman2015Certifying}, authors further studied how to remove disparate impact by modifying the distribution of the unprotected attributes such that the protected attribute cannot be estimated from the unprotected attributes. Research in \cite{Zhang2017Achieving} proposed a causal graph based approach that removes discrimination based on the block set and ensures that there is no discrimination in any meaningful partition. For the in-process approaches, some tweak or regularizers are applied to the classifier to penalize discriminatory prediction during the training process. In principle, preventing discrimination when training a classifier consists of balancing two contrasting objectives: maximizing the accuracy of the extracted predictive model and minimizing the number of predictions that are discriminatory. Research in \cite{Dwork2011Fairness} proposed a predictive model for maximizing utility subject to the fair constraint that achieves both statistical parity and individual fairness, i.e., similar individuals should be treated similarly. In \cite{Hardt2016Equality}, authors proposed a framework for optimally adjusting any predictive model so as to remove discrimination.

Recently, several studies have been proposed to remove discrimination through adversarial training. Research in \cite{Edwards2015Censoring} incorporated an adversarial model to learn a discrimination free representation. Based on that, research in \cite{Beutel2017Data}  studies how the choice of data for the adversarial training affects the fairness. Studies in \cite{Madras2018Learning,Zhang2018Mitigating} further proposed various adversarial objectives to achieve different levels of group fairness including demographic parity, equalized odds and equal opportunity. This paper studies how to generate a discrimination free dataset while still preserving the data generation utility. Fair data generation is in line with the pre-process approaches. The classical pre-process methods like Massaging cannot remove disparate treatment and disparate impact, and the certifying framework, which can remove the disparate impact, can only apply on numerical attributes. On the contrary, FairGAN can remove the disparate treatment and disparate impact from both numerical and categorical data. Meanwhile, compared with the pre-process methods, FairGAN can generate more data for training predictive models, especially when the original training data is very limited.

\section{Preliminary}

\subsection{Fairness and Discrimination}

In fairness-aware learning, the literature has studied notions of group fairness on dataset and classification \cite{Hardt2016Equality,Kamiran2012Data}. 

\begin{definition}[Statistical Parity/Fairness in a Labeled Dataset]
	Given a labeled dataset $\mathcal{D}$, the property of statistical parity or fairness in the labeled dataset is defined as:
	\[
		P(y=1|s=1)=P(y=1|s=0)
	\]
\end{definition}

The discrimination in a labeled dataset w.r.t the protected attribute $\mathcal{S}$ is evaluated by the risk difference: $disc(\mathcal{D}) = P(y=1|s=1)-P(y=1|s=0)$.

\begin{definition}[Statistical Parity/Fairness in a Classifier]
	Given a labeled dataset $\mathcal{D}$ and a classifier $\eta: \mathcal{X} \rightarrow \mathcal{Y}$, the property of statistical parity or fairness in a classifier is defined as:
	\[
		P(\eta(\mathbf{x})=1|s=1)=P(\eta(\mathbf{x})=1|s=0)
	\]
\end{definition}

We can then derive the \textit{discrimination in a classifier} in terms of risk difference as $disc(\eta) = P(\eta(\mathbf{x})=1|s=1)-P(\eta(\mathbf{x})=1|s=0)$.

The classification fairness on a dataset is achieved if both the \textit{disparate treatment} and \textit{disparate impact} are removed from the data. 
To remove the disparate treatment, the classifier cannot use the protected attribute to make decisions. As for the disparate impact,
research in \cite{Feldman2015Certifying} proposed the concept of $\epsilon$-fairness to examine the potential disparate impact. 

\begin{definition}[$\epsilon$-fairness \cite{Feldman2015Certifying}]
	A labeled dataset $\mathcal{D}=(\mathcal{X}, \mathcal{Y}, \mathcal{S})$ is said to be $\epsilon$-fair if for any classification algorithm $f: \mathcal{X} \rightarrow \mathcal{S}$
	\[	
		BER(f(\mathcal{X}), \mathcal{S}) > \epsilon
	\]
	with empirical probabilities estimated from $\mathcal{D}$, where $BER$ (balanced error rate) is defined as
	\[
		BER(f(\mathcal{X}), \mathcal{S}) = \frac{P[f(\mathcal{X})=0|\mathcal{S}=1]+P[f(\mathcal{X})=1|\mathcal{S}=0]}{2}.
	\]
	$BER$  indicates the average class-conditioned error of $f$ on distribution $\mathcal{D}$ over the pair $(\mathcal{X}, \mathcal{S})$.
\end{definition}

The $\epsilon$-fairness quantifies the fairness of data through the error rate of predicting the protected attribute $\mathcal{S}$ given the unprotected attributes $\mathcal{X}$. If the error rate is low, it means $\mathcal{S}$ is predictable by $\mathcal{X}$. 
In the fair data generation scenario, for a classifier trained on the synthetic dataset and tested on the real dataset,
the classification fairness is achieved if  disparate impact in terms of the real protected attribute is removed from the synthetic dataset, i.e. $\hat{\mathcal{X}}\independent \mathcal{S}$. 

\subsection{Generative Adversarial Network}

Generative adversarial nets (GAN) are generative models that consist of two components: a generator $G$ and a discriminator $D$. Typically, both $G$ and $D$ are multilayer neural networks. $G(\mathbf{z})$ generates fake samples from a prior distribution $P_{\mathbf{z}}$ on a noise variable $\mathbf{z}$ and learns a generative distribution $P_G$ to match the real data distribution $P_{\text{data}}$. The discriminative component $D$ is a binary classifier that predicts whether an input is real data $\mathbf{x}$ or fake data generated from $G(\mathbf{z})$. The objective function of $D$ is defined as:
\begin{equation}
\label{eq:reg_d_loss}
\max\limits_D \quad \mathbb{E}_{\mathbf{x} \sim P_{\text{data}}}[\log D(\mathbf{x})] + \mathbb{E}_{\mathbf{z} \sim P_{\mathbf{z}}}[\log(1-D(G(\mathbf{z})))],
\end{equation}
where $D(\cdot)$ outputs the probability that $\cdot$ is from the real data rather than the generated fake data. In order to make the generative distribution $P_G$ close to the real data distribution $P_{\text{data}}$, $G$ is trained by fooling the discriminator unable to distinguish the generated data from the real data. Thus, the objective function of $G$ is defined as:
\begin{equation}
\label{eq:reg_g_loss}
\min\limits_G \quad \mathbb{E}_{\mathbf{z} \sim P_{\mathbf{z}}}[\log(1-D(G(\mathbf{z})))].
\end{equation}
Minimization of Equation \ref{eq:reg_g_loss} ensures that the discriminator is fooled by $G(\mathbf{z})$ and  $D$  predicts high probability that $G(\mathbf{z})$ is real data.

Overall, GAN is formalized as a minimax game $\min\limits_G \max\limits_D V(G,D)$ with the value function:
\begin{equation}
V(G,D)= \mathbb{E}_{\mathbf{x} \sim P_{\text{data}}}[\log D(\mathbf{x})] + \mathbb{E}_{\mathbf{z} \sim P_{\mathbf{z}}}[\log(1-D(G(\mathbf{z})))].
\end{equation}

Figure \ref{fig:gan} illustrates the structure of GAN. Theoretical analysis shows that GAN aims to minimize the Jensen-Shannon divergence (JSD) between $P_{\text{data}}$ and $P_G$ \cite{Goodfellow2014Generative}. Minimization of the JSD is achieved when $P_G = P_{\text{data}}$. 

\begin{figure}
	\centering
	\includegraphics[width=0.38\textwidth]{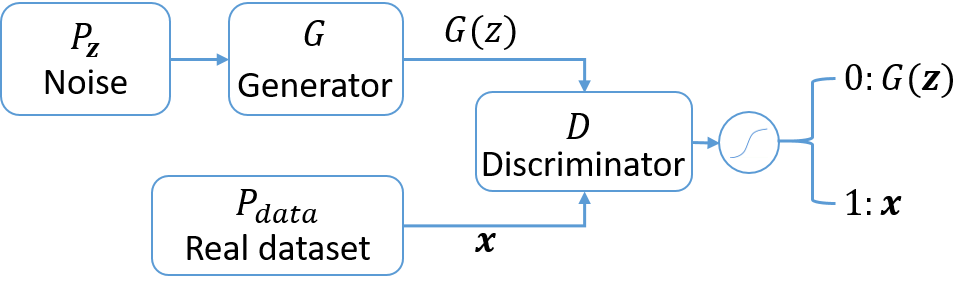}
	\caption{Illustration of generative adversarial networks}
	\label{fig:gan}
\end{figure}

{\bf \noindent GAN for discrete data generation.}
The generator of a regular GAN cannot generate discrete samples because $G$ is trained by the loss from  $D$ via backpropagation \cite{Goodfellow2014Generative}. In order to tackle this limitation, medGAN incorporates an autoencoder model in a regular GAN model to generate high-dimensional discrete variables \cite{Choi2017Generating}. 


Autoencoder is a feedforward neural network used for unsupervised learning. A basic autoencoder consists of an encoder $Enc$ and a decoder $Dec$. Both the encoder and decoder are multilayer neural networks. Given an input $\mathbf{x} \in \mathbb{R}^n$, the encoder computes the hidden representation of an input $Enc(\mathbf{x}) \in \mathbb{R}^h$, and the decoder computes the reconstructed input $Dec(Enc(\mathbf{x})) \in \mathbb{R}^n$ based on the hidden representation. To train the autoencoder model, the objective function of the autoencoder is to make the reconstructed input close to the original input:
\begin{equation}
	\mathcal{L}_{AE} = ||\mathbf{x}' - \mathbf{x}||_2^2,
\end{equation}
where $\mathbf{x}' = Dec(Enc(\mathbf{x}))$. Because the hidden representation can be used to reconstruct the original input, it captures the salient information of the input.

To generate the dataset which contains discrete attributes, the generator $G_{Dec}$ in medGAN consists of two components, the  generator $G$ and the decoder $Dec$. The generator $G$ is trained to generate the salient representations. The decoder $Dec$ from autoencoder seeks to construct the synthetic data from the salient representations $Dec(G(\mathbf{z}))$. Hence, the generator of medGAN $G_{Dec}(\mathbf{z})$ is defined as:
\[
	G_{Dec}(\mathbf{z}) = Dec(G(\mathbf{z})),
\]
where $\mathbf{z}$ is a noise variable. The discriminator $D$ aims to distinguish whether the input is from real data or $Dec(G(\mathbf{z}))$. With playing the adversarial game with $D$, the gradient backpropagates from $D$ to the decoder $Dec$ and further to the generator $G$ in an end-to-end manner. Hence, the generator $G_{Dec}$ can be viewed as a regular generator $G$ with extra hidden layers that maps continuous salient representations to discrete samples. 

\section{FairGAN}

\begin{figure}
	\centering
	\includegraphics[width=0.32\textwidth]{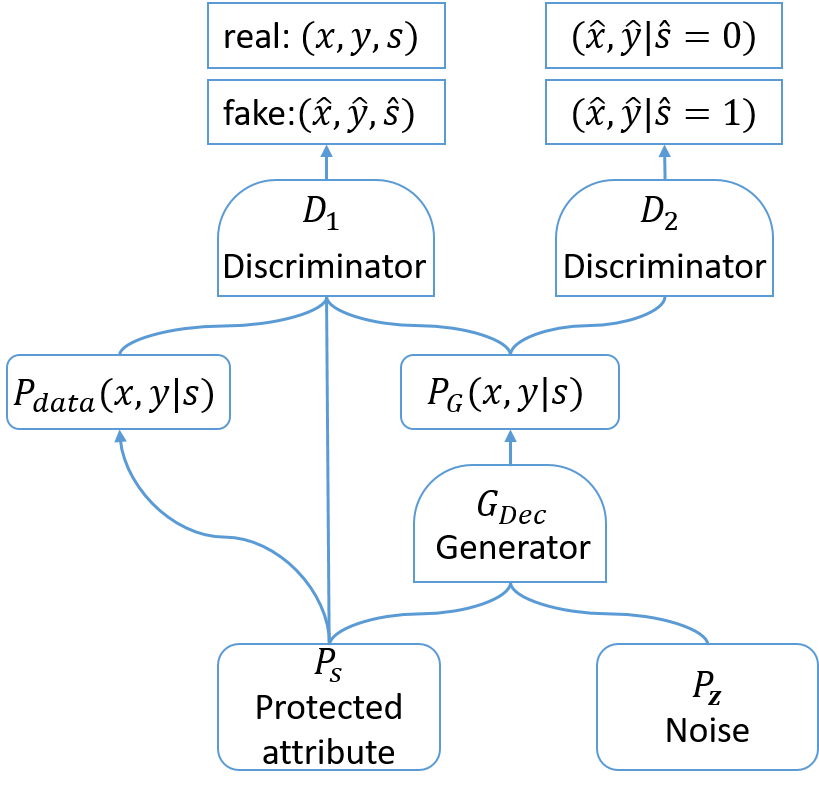}
	\caption{The Structure of FairGAN}
	\label{fig:fairgan}
\end{figure}

\subsection{Problem Statement}

Given a dataset $\{\mathcal{X}, \mathcal{Y}, \mathcal{S}\} \sim P_{\text{data}}$, FairGAN aims to generate a fair dataset $\{\hat{\mathcal{X}}, \hat{\mathcal{Y}}, \hat{\mathcal{S}}\} \sim P_{G}$ which achieves the \textit{statistical parity} w.r.t the protected attribute $\hat{\mathcal{S}}$, i.e., $P(\hat{y}=1|\hat{s}=1) = P(\hat{y}=1|\hat{s}=0)$. Meanwhile, our goal is to ensure that given a generated dataset $\{\hat{\mathcal{X}}, \hat{\mathcal{Y}}\}$ as training samples, a classification model seeks an accurate function $\eta: \hat{\mathcal{X}} \rightarrow \hat{\mathcal{Y}}$ while satisfying fair classification with respect to the protected attribute on the real dataset, i.e., $P(\eta(\mathbf{x})=1|s=1)=P(\eta(\mathbf{x}=1)|s=0)$. 

\subsection{Model}

FairGAN consists of one generator $G_{Dec}$ and two discriminators $D_1$ and $D_2$. We adopt the revised generator from medGAN \cite{Choi2017Generating} to generate both discrete and continuous data. Figure \ref{fig:fairgan} shows the structure of FairGAN. 
In FairGAN, every generated sample has a corresponding value of the protected attribute $s \sim P_{\text{data}}(s)$. The generator $G_{Dec}$ generates a fake pair $(\hat{\mathbf{x}},\hat{y})$ following the conditional distribution $P_{G}(\mathbf{x},y|s)$. The fake pair $(\hat{\mathbf{x}},\hat{y})$ is generated by a noise variable $\mathbf{z}$ given the protected attribute $s$, namely, 
\begin{equation}
\label{eq:generator}
\hat{\mathbf{x}},\hat{y} = G_{Dec}(\mathbf{z},s), \mathbf{z} \sim P_{\mathbf{z}}(\mathbf{z}), 
\end{equation}
where $P_{\mathbf{z}}(\mathbf{z})$ is a prior distribution. Hence, the generated fake sample $(\hat{\mathbf{x}},\hat{y},\hat{s})$ is from the joint distribution $P_{G}(\mathbf{x},y,s)=P_{G}(\mathbf{x},y|s)P_{G}(s)$, where $P_G(s)=P_{\text{data}}(s)$. 
The discriminator $D_1$ is trained to distinguish between the real data from $P_{\text{data}}(\mathbf{x}, y, s)$ and the generated fake data from $P_G(\mathbf{x},y,s)$.


Meanwhile, in order to make the generated dataset achieve fairness, a constraint is applied to the generated samples, which aims to keep $P_{G}(\mathbf{x},y|s=1) = P_{G}(\mathbf{x},y|s=0)$. Therefore, another discriminator $D_2$ is incorporated into the FairGAN model and trained to distinguish the two categories of generated samples, $P_{G}(\mathbf{x},y|s=1)$ and $P_{G}(\mathbf{x},y|s=0)$. 

The value function of the minimax game is described as:
\begin{equation}
\label{eq:fairgan}
		 \min_{G_{Dec}} \max_{D_1,D_2}  V(G_{Dec}, D_1, D_2) =  V_1(G_{Dec}, D_1) + \lambda V_2(G_{Dec}, D_2),  
\end{equation}
where 
\begin{equation}
\begin{split}
\label{eq:gan}
	& ~V_1  (G_{Dec},  D_1) \\ & = \mathbb{E}_{s \sim P_{\text{data}}(s), (\mathbf{x},y) \sim P_{\text{data}}(\mathbf{x},y|s)} {[\log D_1(\mathbf{x}, y, s)]} \\
		 & + \mathbb{E}_{\hat{s} \sim P_{G}(s), (\hat{\mathbf{x}},\hat{y}) \sim P_{G}(\mathbf{x},y|s)} {[\log (1 - D_1(\hat{\mathbf{x}},\hat{y}, \hat{s}))]},
\end{split}
\end{equation}


\begin{equation}
\begin{split}
\label{eq:fair}
	V_2(G_{Dec}, D_2) & = \mathbb{E}_{(\hat{\mathbf{x}},\hat{y}) \sim P_G (\mathbf{x},y|s=1)}[\log D_2(\hat{\mathbf{x}},\hat{y})] \\
		 & + \mathbb{E}_{(\hat{\mathbf{x}},\hat{y}) \sim P_G (\mathbf{x},y|s=0)}[\log (1-D_2(\hat{\mathbf{x}},\hat{y}))],
\end{split}
\end{equation}
$\lambda$ is a hyperparameter that specifies a trade off between utility and fairness of data generation.


The first value function $V_1$ is similar to a conditional GAN model \cite{Mirza2014Conditional}, where the generator $G$ seeks to learn the joint distribution $P_G(\mathbf{x},y,s)$ over real data $P_{\text{data}}(\mathbf{x},y,s)$ by first drawing $\hat{s}$ from $P_{G}(s)$ and then drawing $\{\hat{\mathbf{x}},\hat{y}\}$ from $P_G(\mathbf{x},y|s)$ given a noise variable. Note that in the generated sample $\{\hat{\mathbf{x}},\hat{y}, \hat{s}\}$, the protected attribute $\hat{s}=s$ due to the generator conditioning on $s$ to generate $\{\hat{\mathbf{x}},\hat{y}\}$. The second value function $V_2$ aims to make the generated samples not encode any information supporting to predict the value of protected attribute $s$. Therefore, $D_2$ is trained to correctly predict $s$ given a generated sample while the generator $G$ aims to fool the discriminator $D_2$. Once the generated sample $\{\hat{\mathbf{x}},\hat{y}\}$ cannot be used to predict the protected attribute $\hat{s}$ ($s$), the correlation between $\{\hat{\mathbf{x}},\hat{y}\}$ and $s$ is removed, i.e., $\{\hat{\mathbf{x}},\hat{y}\} \independent s$. FairGAN can ensure that the generated samples do not have the disparate impact. 

For the decoder $Dec$ to convert the representations to data samples, FairGAN first pre-trains the autoencoder model. The decoder then can generate samples given the representation from $G(\mathbf{z},s)$. Meanwhile, since the autoencoder is pre-trained by the original dataset that may contain discrimination information, we further fine-tune the decoder $Dec$ to remove the discrimination information when optimizing the generator $G$. 
The procedure of training FairGAN is shown in Algorithm \ref{algr:train}, where $\theta_{ae}$, $\theta_{d_1}$, $\theta_{d_2}$ and $\theta_{g}$ are trainable parameters in autoencoder, $D_1$, $D_2$ and $G_{Dec}$ respectively. FairGAN first pretrains the autoencoder (from Line \ref{algr:start_pre} to \ref{algr:end_pre}). For training the generator $G_{Dec}$ and discriminators $D_1$ and $D_2$, FairGAN first samples a batch of real data and a batch of fake data to train $G_{Dec}$ and $D_1$ (from Line \ref{algr:start_d1} to \ref{algr:end_d1}) and then applies the fair constraint to train $G_{Dec}$ and $D_2$ (from Line \ref{algr:start_d2} to \ref{algr:end_d2}).

\begin{algorithm}[htb]
\begin{algorithmic}[1]
	\For{number of pre-training iterations}
	\label{algr:start_pre}
		\State Sample a batch of $m$ examples $(\mathbf{x}, y, s) \sim P_{\text{data}}(\mathbf{x}, y, s)$
		\State Update Autoencoder by descending its stochastic gradient:
		\[\nabla_{\theta_{ae}} \frac{1}{m}\sum_{i=1}^m ||\mathbf{x}' - \mathbf{x}||_2^2\]
	\EndFor
	\label{algr:end_pre}

	\For{number of training iterations}
		\State Sample a batch of $m$ examples $(\mathbf{x}, y, s) \sim P_{\text{data}}(\mathbf{x}, y, s)$
		\label{algr:start_d1}
		\State Sample a batch of $m$ examples $(\hat{\mathbf{x}}, \hat{y}, \hat{s}) \sim P_{G}(\mathbf{x}, y, s)$ from generator $G_{Dec}(\mathbf{z}, s)$ by first drawing $s \sim P_{G}(s)$ and noise samples $\mathbf{z} \sim P_{\mathbf{z}}(\mathbf{z})$

		\State Update $D_1$ by ascending its stochastic gradient:
		\[\nabla_{\theta_{d_1}} \frac{1}{m}\sum_{i=1}^m \big[\log D_1(\mathbf{x}, y, s) + \log(1-D_1(\hat{\mathbf{x}}, \hat{y}, \hat{s}))\big]\]

		\State Update $G_{Dec}$ by descending along its stochastic gradient:
		\[\nabla_{\theta_{g}} \frac{1}{m}\sum_{i=1}^m \log(1-D_1(\hat{\mathbf{x}}, \hat{y}, \hat{s}))\]
		\label{algr:end_d1}
		\State Sample a batch of $m$ examples $(\hat{\mathbf{x}}, \hat{y} | \hat{s}=1) \sim P_{G}(\mathbf{x}, y|s=1)$ and sample another batch of $m$ examples $(\hat{\mathbf{x}}, \hat{y} | \hat{s}=0) \sim P_{G}(\mathbf{x}, y|s=0)$
		\label{algr:start_d2}
		\State Update $D_2$ by ascending its stochastic gradient:
		\[\nabla_{\theta_{d_2}} \frac{1}{2m}\sum_{i=1}^{2m} \big[\log D_2(\hat{\mathbf{x}}, \hat{y}) + \log(1-D_2(\hat{\mathbf{x}}, \hat{y}))\big]\]

		\State Update $G_{Dec}$ by descending along its stochastic gradient:
		\[\nabla_{\theta_{g}} \frac{1}{2m}\sum_{i=1}^{2m} \big[\log D_2(\hat{\mathbf{x}}, \hat{y}) + \log(1-D_2(\hat{\mathbf{x}}, \hat{y}))\big]\]
		\label{algr:end_d2}
	\EndFor
\end{algorithmic}
\caption{Minibatch stochastic gradient descent training of FairGAN.}
\label{algr:train}
\end{algorithm}

\subsection{Theoretical analysis}

We conduct the similar theoretical analysis as the traditional GAN model \cite{Goodfellow2014Generative}. We first consider the optimal discriminators $D_1$ and $D_2$ for any given generator $G$.

\begin{prop}\label{prop:fairgan}
For $G$ fixed, the optimal discriminators $D_1$ and $D_2$ by the value function $V(G_{Dec}, D_1, D_2)$ are
\[
	D_1^*(\mathbf{x},y,s) = \frac{P_{\text{data}}(\mathbf{x},y,s)} {P_{\text{data}}(\mathbf{x},y,s)+P_{G}(\mathbf{x},y,s)},
\]
\[
	D_2^*(\mathbf{x},y) = \frac{P_{G}(\mathbf{x},y|s=1)}{P_{G}(\mathbf{x},y|s=1)+P_{G}(\mathbf{x},y|s=0)}
\]
\end{prop}

\begin{proof}
The training criteria for the discriminators $D_1$ and $D_2$, given any generator $G_{Dec}$, are to maximize  $V(G_{Dec}, D_1, D_2)$:\\
\resizebox{.9\linewidth}{!}{
\begin{minipage}{\linewidth}
\begin{equation}
\begin{split}
	&~V(G_{Dec}, D_1, D_2)  \\
		 & = \int_x \int_y   P_{\text{data}}(\mathbf{x},y,s) \log D_1(\mathbf{x},y,s) 
		  + \int_x \int_y P_{G}(\mathbf{x},y,s) \log (1-D_1(\mathbf{x},y,s)) \\
		 & + \lambda \int_x \int_y  P_{G}(\mathbf{x},y|s=1) \log (D_2(\mathbf{x},y)) 
		 + \lambda \int_x \int_y  P_{G}(\mathbf{x},y|s=0) \log (1 - D_2(\mathbf{x},y))
\end{split}
\end{equation}
\end{minipage}
}

Following \cite{Goodfellow2014Generative}, for any $(a,b) \in \mathbb{R}^2 \setminus \{0,0\}$, the function $y \rightarrow a\log(y)+b\log(1-y)$ achieves its maximum in $[0,1]$ at $\frac{a}{a+b}$. This concludes the proof.
\end{proof}

\begin{theorem}
The minimax game of FairGAN can be reformulated as $C(G_{Dec})=-(2+\lambda) \log4 + 2 \cdot JSD(P_{\text{data}}(\mathbf{x},y,s)||P_{G}(\mathbf{x},y,s))+2 \lambda \cdot JSD(P_{G}(\mathbf{x},y|s=1)||P_{G}(\mathbf{x},y|s=0))$, where $JSD$ indicates the Jensen-Shannon divergence between two distributions. The minimum value is $-(2+\lambda) \log4 + \Delta$, where $\Delta$ is the minimum value when the two JSD terms are converged to a global optimal point due to convexity of JSD.
\end{theorem}

\begin{proof}
Given $D_1^*$ and $D_2^*$, we reformulate Equation \ref{eq:fairgan} as:

\resizebox{.98\linewidth}{!}{
\begin{minipage}{\linewidth}
\begin{equation}
\begin{split}
\label{eq:maxd}
   	&~~C(G_{Dec}) \\
	& = \max_{D_1,D_2} V(G_{Dec}, D_1, D_2) \\
   & = \mathbb{E}_{(\mathbf{x},y,s) \sim P_{\text{data}}(\mathbf{x},y,s)}[\log \frac{P_{\text{data}}(\mathbf{x},y,s)} {P_{\text{data}}(\mathbf{x},y,s)+P_{G}(\mathbf{x},y,s)}] \\
   & + \mathbb{E}_{(\mathbf{x},y,s) \sim P_{G}(\mathbf{x},y,s)}[\log \frac{P_{G}(\mathbf{x},y,s)} {P_{\text{data}}(\mathbf{x},y,s)+P_{G}(\mathbf{x},y,s)}] \\
   & + \lambda \mathbb{E}_{(\mathbf{x},y) \sim P_{G}(\mathbf{x},y|s=1)} [\log \frac{P_{G}(\mathbf{x},y|s=1)} {P_{G}(\mathbf{x},y|s=1)+P_{G}(\mathbf{x},y|s=0)}] \\
   & + \lambda \mathbb{E}_{(\mathbf{x},y) \sim P_{G}(\mathbf{x},y|s=0)} [\log \frac{P_{G}(\mathbf{x},y|s=0)} {P_{G}(\mathbf{x},y|s=1)+P_{G}(\mathbf{x},y|s=0)}] \\
   & =  -(2+\lambda) \log4 + 2 \cdot JSD(P_{\text{data}}(\mathbf{x},y,s)||P_{G}(\mathbf{x},y,s)) \\
   & + 2 \lambda \cdot JSD(P_{G}(\mathbf{x},y|s=1)||P_{G}(\mathbf{x},y|s=0)),
\end{split}
\end{equation}
\end{minipage}
}
While the two pairs of distributions $P_{\text{data}}(\mathbf{x},y,s)=P_{G}(\hat{\mathbf{x}},\hat{y},s)$ and $P_{G}(\hat{\mathbf{x}},\hat{y}|s=1)=P_{G}(\hat{\mathbf{x}},\hat{y}|s=0)$ cannot be achieved simultaneously in Equation \ref{eq:maxd}, the two JSD terms can still be converged to a global optimal point due to convexity of JSD. Hence, the optimum value of $V(G_{Dec}, D_1, D_2) = -(2+\lambda) \log4 + \Delta$, where $\Delta$ is the minimum value when the two JSD terms are converged to a point. 

\end{proof}

\subsection{Na\"iveFairGAN models}
In this subsection, we discuss two naive approaches which can only achieve fair data generation (disparate treatment) but cannot achieve fair classification (disparate impact). 

{\bf \noindent Na\"iveFairGAN-I}

To mitigate the disparate treatment, a straightforward approach is to remove $\mathcal{S}$ from the dataset. Hence, if a GAN model ensures the generated samples have the same distribution as the real data with unprotected attributes and decision, i.e., $P_G(\mathbf{x},y)=P_{\text{data}}(\mathbf{x},y)$, and randomly assigns the values of protected attribute with only preserving the ratio of protected group to unprotected group the same as the real data, the completely generated dataset could achieve the statistical parity in the dataset. Because there is no additional fair constraint in data generation, the Na\"iveFairGAN-I model is a regular GAN model which consists of one generator and one discriminator. The value function of Na\"iveFairGAN-I is defined as:

\begin{equation*}
\begin{split}
	\min_{G_{Dec}} \max_{D}  V(G_{Dec}, D)  &= \mathbb{E}_{\mathbf{x},y \sim P_{\text{data}}(\mathbf{x},y)}{[\log D_1(\mathbf{x}, y)]} \\
	& + \mathbb{E}_{\hat{\mathbf{x}},\hat{y} \sim P_{G}(\mathbf{x},y)}{[\log (1 - D_1(\hat{\mathbf{x}},\hat{y}))]},
\end{split}
\end{equation*}

In principle, Na\"iveFairGAN-I achieves the fair data generation by \textit{randomly generating} the protected attribute $\hat{\mathcal{S}}$. However, due to the property $P_G(\mathbf{x}) = P_{\text{data}}(\mathbf{x})$, the disparate impact caused by the correlation between generated unprotected attributes $\hat{\mathcal{X}}$ and the real protected attribute $\mathcal{S}$ is not removed. The classifier trained on the generated dataset cannot achieve fair prediction when tested on real data. 

{\bf \noindent Na\"iveFairGAN-II}

We can extend the Na\"iveFairGAN-I to Na\"iveFairGAN-II by considering the protected attribute. Na\"iveFairGAN-II consists of one generator and two discriminators. Given the protected attribute $\mathcal{S}$, the real data can be seen as from two sub-domains, $P_{\text{data}}(\mathbf{x},y|s=1)$ and $P_{\text{data}}(\mathbf{x},y|s=0)$. The discriminator $D_1$ is used to detect whether a pair of sample $(\mathbf{x},y)$ is from the real data. If the sample is from the generator, then another discriminator $D_2$ is used to distinguish whether this sample is from $P_G (\mathbf{x},y|s=0)$ or $P_G (\mathbf{x},y|s=1)$. Hence, by playing this three-player game with value function shown in Equation \ref{eq:nafairgan}, the first discriminator is to ensure the generated sample be close to the real sample while the second discriminator is to ensure the demographic parity of the generated data. 

\begin{equation}
\label{eq:nafairgan}
	\min_{G_{Dec}} \max_{D_1,D_2}  V(G_{Dec}, D_1, D_2) =  V_1(G_{Dec}, D_1) + V_2(G_{Dec}, D_2)
\end{equation}
where
\begin{equation*}
\begin{split}
	 V_1 (G_{Dec}, D_1) &=  \mathbb{E}_{(\mathbf{x},y) \sim P_{\text{data}}(\mathbf{x},y|s=1)} {[\log D_1(\mathbf{x}, y)]} \\ 
	&+ \mathbb{E}_{(\mathbf{x},y) \sim P_{\text{data}}(\mathbf{x},y|s=0)} {[\log D_1(\mathbf{x}, y)]} \\
		 & + \mathbb{E}_{(\hat{\mathbf{x}},\hat{y}) \sim P_{G}(\mathbf{x},y|s=1)} {[\log (1 - D_1(\hat{\mathbf{x}},\hat{y}))]} \\
		 & + \mathbb{E}_{(\hat{\mathbf{x}},\hat{y}) \sim P_{G}(\mathbf{x},y|s=0)} {[\log (1 - D_1(\hat{\mathbf{x}},\hat{y}))]} ,
\end{split}
\end{equation*}

\begin{equation*}
\begin{split}
	V_2(G_{Dec}, D_2) & = \mathbb{E}_{(\hat{\mathbf{x}},\hat{y}) \sim P_G (\mathbf{x},y|s=1)}[\log D_2(\hat{\mathbf{x}},\hat{y})] \\
			 & + \mathbb{E}_{(\hat{\mathbf{x}},\hat{y}) \sim P_G (\mathbf{x},y|s=0)}[\log (1-D_2(\hat{\mathbf{x}},\hat{y}))],
\end{split}
\end{equation*}

Following the similar theoretical analysis as FairGAN, the global minimum of Na\"iveFairGAN-II $V(G_{Dec}, D_1, D_2)$ is achieved if and only if $P_{G}(\mathbf{x},y|s=1) = P_{G}(\mathbf{x},y|s=0) = \big(P_{\text{data}}(\mathbf{x},y|s=1)+P_{\text{data}}(\mathbf{x},y|s=0)\big)/2$ with $D_1^*=\frac{1}{2}$ and $D_2^*=\frac{1}{2}$. At that point, $V(G_{Dec}, D_1, D_2)$ achieves the value $-4 \log4$. The detailed proof is described in the appendix. 

Although Na\"iveFairGAN-II considers the protected attribute, it achieves $P_{G}(\mathbf{x},y|s=1) = P_{G}(\mathbf{x},y|s=0) = \big(P_{\text{data}}(\mathbf{x},y|s=1)+P_{\text{data}}(\mathbf{x},y|s=0)\big)/2$ without ensuring $P_{G}(\mathbf{x},y|s=1) = P_{\text{data}}(\mathbf{x},y|s=1)$ or $P_{G}(\mathbf{x},y|s=0) = P_{\text{data}}(\mathbf{x},y|s=0)$. Hence, Na\"iveFairGAN-II achieves the fair data generation by \textit{randomly shuffling} the protected attribute $\hat{\mathcal{S}}$. Similar to the data from Na\"iveFairGAN-I, the generated data from Na\"iveFairGAN-II also incurs disparate impact because the correlation between generated unprotected attributes $\hat{\mathcal{X}}$ and the real protected attribute $\mathcal{S}$ is not removed.

{\bf \noindent FairGAN vs. Na\"iveFairGAN-I vs. Na\"iveFairGAN-II}

\begin{figure*}[h]
	\centering 
	\begin{subfigure}{0.23\textwidth}
		\includegraphics[width=\linewidth]{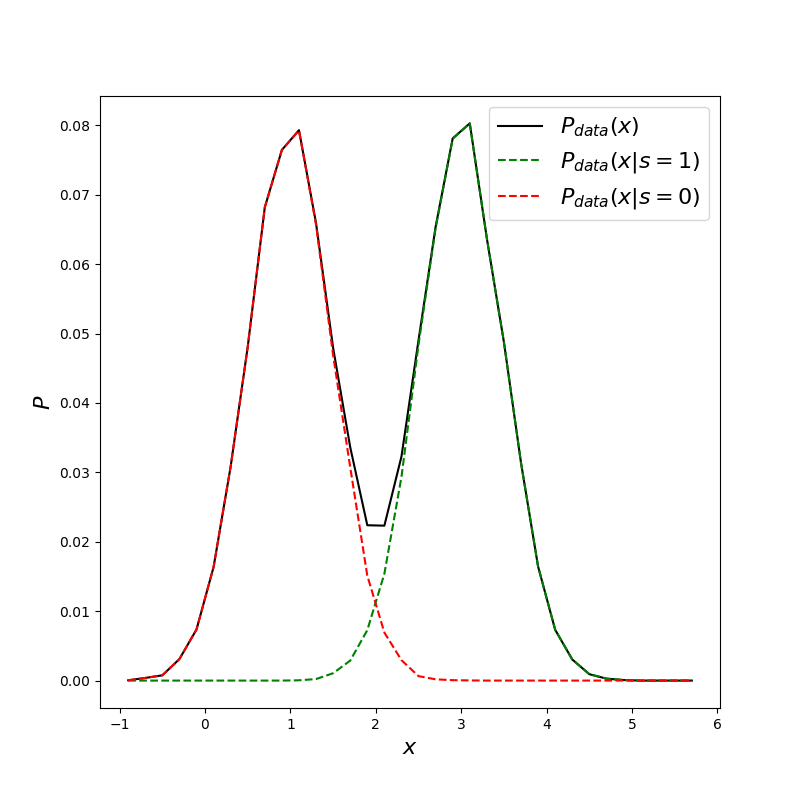}
		\caption{Toy Dataset}
		\label{fig:toy}
	\end{subfigure}\hfil 
	\begin{subfigure}{0.23\textwidth}
		\includegraphics[width=\linewidth]{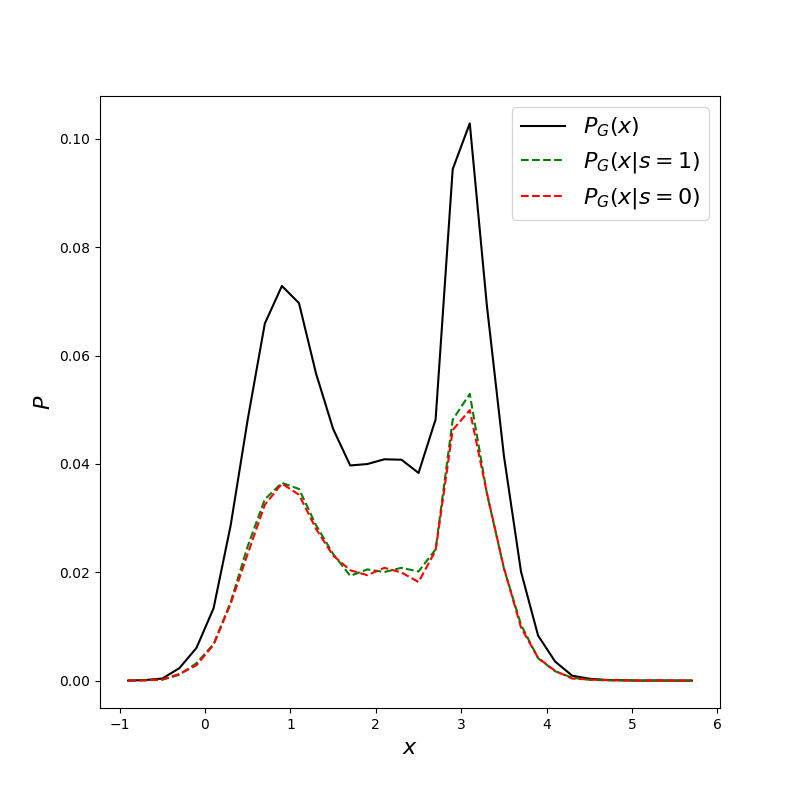}
		\caption{Na\"iveFairGAN-I}
		\label{fig:toy_nfair1}
	\end{subfigure}\hfil 
	\begin{subfigure}{0.23\textwidth}
		\includegraphics[width=\linewidth]{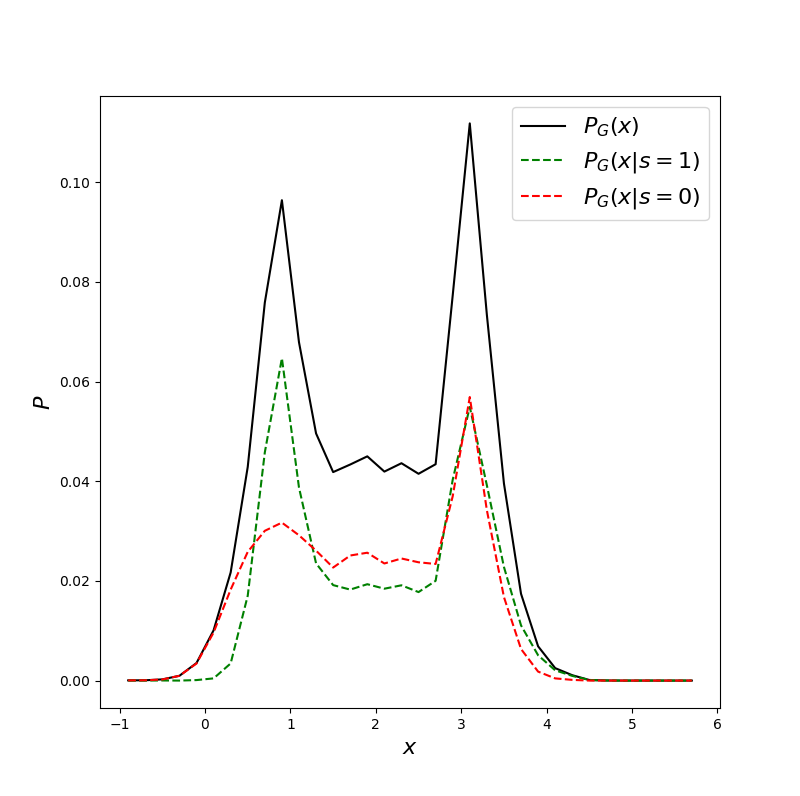}
		\caption{Na\"iveFairGAN-II}
		\label{fig:toy_nfair2}
	\end{subfigure}
	\begin{subfigure}{0.23\textwidth}
		\includegraphics[width=\linewidth]{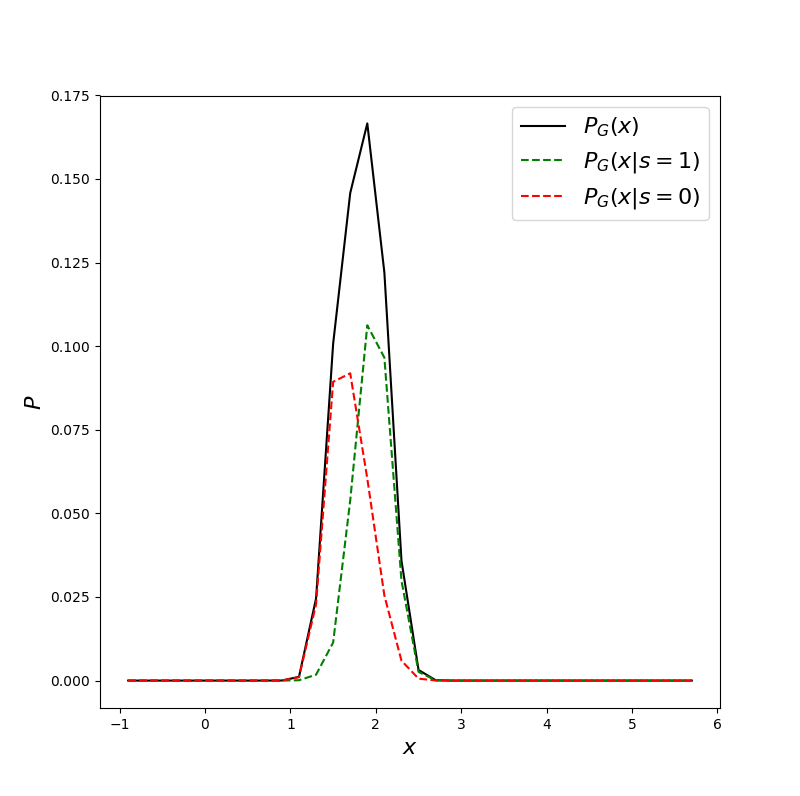}
		\caption{FairGAN}
		\label{fig:toy_fair}
	\end{subfigure}
	\caption{Comparing FairGAN, Na\"iveFairGAN-I and Na\"iveFairGAN-II on a toy dataset. (a) shows the distributions $P_{\text{data}}(x)$ (black), $P_{\text{data}}(x|s=1)$ (green) and $P_{\text{data}}(x|s=0)$ (red) of real data; (b), (c) and (d) are distributions $P_G(x)$, $P_G(x|s=1)$ and $P_G(x|s=0)$ of synthetic datasets generated by Na\"iveFairGAN-I, Na\"iveFairGAN-II and FairGAN separately.}
	\label{fig:toy_all}
\end{figure*}

We compare FairGAN with Na\"iveFairGAN-I and Na\"iveFairGAN-II on a toy dataset which consists of one unprotected attribute $x \in \mathbb{R}$ and one protected attribute $s \in \{0,1\}$. The toy dataset is drawn from $x \sim 0.5*\mathcal{N}(1,0.5)+0.5*\mathcal{N}(3,0.5)$, where $P_{\text{data}}(x|s=1)=\mathcal{N}(1,0.5)$ and $P_{\text{data}}(x|s=0)=\mathcal{N}(3,0.5)$. Hence, the unprotected attribute $x$ is strong correlated with the protected attribute $s$. 

We train FairGAN and Na\"iveFairGAN models to approximate the distribution of $P_{\text{data}}(x)$. Figure \ref{fig:toy_all} shows the data probability $P(x)$ and two conditional probabilities $P(x|s=1)$ and $P(x|s=0)$ of the toy dataset (shown in Figure \ref{fig:toy}) and synthetic datasets (Figures \ref{fig:toy_nfair1} to \ref{fig:toy_fair}) from FairGAN and Na\"iveFairGAN models.

For Na\"iveFairGAN-I, it is a regular GAN model which aims to make $P_G(x)=P_{\text{data}}(x)$ while $s$ is independently generated. Therefore, in this toy example, as shown in Figure \ref{fig:toy_nfair1}, we can observe that $P_G(x)$ is similar to $P_{\text{data}}(x)$. Meanwhile, because $s$ is independently assigned instead of generated from the GAN model,  $P_G(x|s=1)$ and $P_G(x|s=0)$ are almost identical to each other, which avoids disparate treatment. However, due to the high similarity between $P_G(x)$ and $P_{\text{data}}(x)$, given the generated $\hat{x}$ and the real  $s$, the potential disparate impact isn't mitigated. 

For Na\"iveFairGAN-II, in the ideal situation, it has $P_G(x|s=1)=P_G(x|s=0)= \big(P_{\text{data}}(x|s=1)+P_{\text{data}}(x|s=0)\big)/2$. As shown in Figure \ref{fig:toy_nfair2}, the empirical results are close to the theoretical analysis. Hence, the generated dataset from Na\"iveFairGAN-II achieves statistical parity. However, similar to the dataset from Na\"iveFairGAN-I, the dataset from Na\"iveFairGAN-II also has $P_G(x)=P_{\text{data}}(x)$, so the disparate impact in terms of the real $s$ isn't mitigated. 

For FairGAN, in this toy example, the discriminator $D_1$ in FairGAN aims to ensure $P_G(x, s)=P_{\text{data}}(x, s)$. For $s=1$, we have $P_G(s=1)P_G(x|s=1)=P_{\text{data}}(s=1)P_{\text{data}}(x|s=1)$. Since $P_G(s=1)=P_{\text{data}}(s=1)$, FairGAN ensures $P_G(x|s=1)=P_{\text{data}}(x|s=1)$. Similarly, FairGAN also ensures $P_G(x|s=0)=P_{\text{data}}(x|s=0)$. Hence,  $D_1$ in FairGAN makes the two conditional distributions of synthetic data are close to those of the real data. Meanwhile, the discriminator $D_2$ in FairGAN aims to make $P_G(x|s=1)=P_{G}(x|s=0)$. The statistical parity of synthetic data is achieved by making two conditional distributions close to each other. Figure \ref{fig:toy_fair} shows that $P_G(x|s=1)$ and $P_G(x|s=0)$ are close to Gaussian distributions, especially compared with the data distributions from two Na\"iveFairGAN models. In addition, the mean values of $P_G(x|s=1)$ and $P_G(x|s=0)$ are similar, which ensures the real $s$ is unpredictable given the generated  $\hat{x}$, i.e., free from disparate impact.

For fair data generation, the statistical parity in datasets can be achieved by adopting FairGAN and Na\"iveFairGAN models. For fair classification on the real data, we should remove both disparate treatment and disparate impact. The disparate treatment can be removed by shuffling or removing the protected attribute while the disparate impact can be removed only when correlations between unprotected attributes and the protected attribute are removed by modifying the unprotected attributes \cite{Feldman2015Certifying}. Especially, in data generation scenario, the generated $\hat{\mathcal{X}}$ should be unable to predict the real $\mathcal{S}$. For Na\"iveFairGAN-I and Na\"iveFairGAN-II, $P_G(x)$ are close to $P_{\text{data}}(x)$. The potential disparate impact  is still preserved since $\hat{\mathcal{X}}$ have correlations with the real $\mathcal{S}$. A classifier trained on the synthetic dataset from Na\"iveFairGAN-I or Na\"iveFairGAN-II would capture the disparate impact, so the statistical parity in the classifier cannot be achieved when predicting on the real dataset. On the contrary, FairGAN revises $P_G(x)$  to ensure $P_G(x|s=1)=P_G(x|s=0)$, and the protected attribute from FairGAN is the same as the real data, $\hat{\mathcal{S}} = \mathcal{S}$. Once the correlations between $\hat{\mathcal{X}}$ and $\hat{\mathcal{S}}$ are removed, the correlations between $\hat{\mathcal{X}}$ and $\mathcal{S}$ are also removed. A classifier trained on the synthetic dataset from FairGAN can achieve statistical parity when tested on the real dataset.

\begin{figure*}[h!]
	\centering 
	\begin{subfigure}{0.19\textwidth}
		\includegraphics[width=\linewidth]{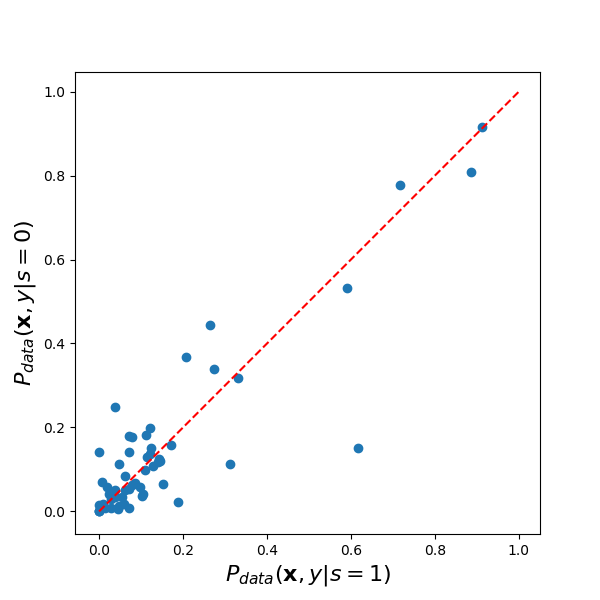}
		\caption{Real Dataset}
		\label{fig:f0}
	\end{subfigure}\hfil 
	\begin{subfigure}{0.19\textwidth}
		\includegraphics[width=\linewidth]{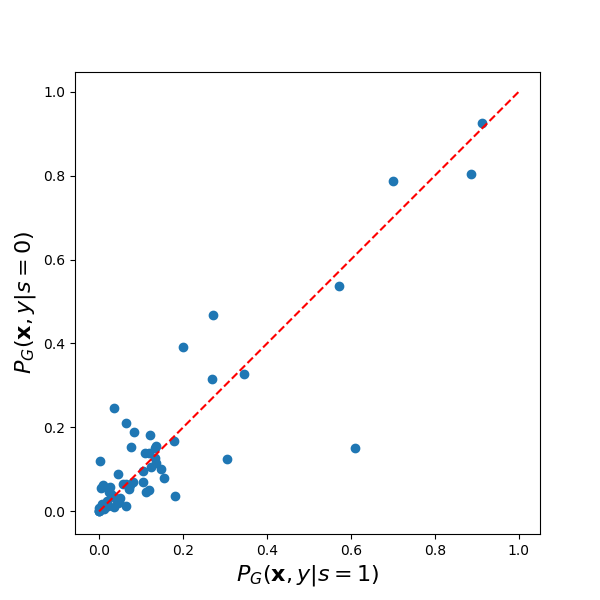}
		\caption{SYN1-GAN}
		\label{fig:f1}
	\end{subfigure}\hfil 
	\begin{subfigure}{0.19\textwidth}
		\includegraphics[width=\linewidth]{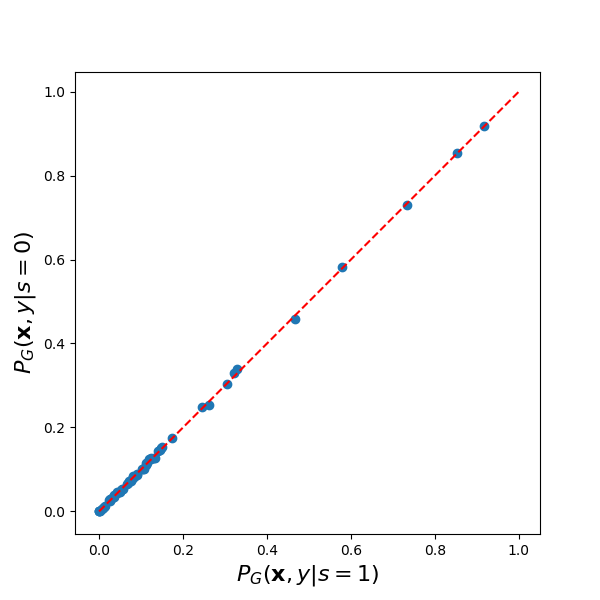}
		\caption{SYN2-NFGANI}
		\label{fig:f2}
	\end{subfigure}
	\begin{subfigure}{0.19\textwidth}
		\includegraphics[width=\linewidth]{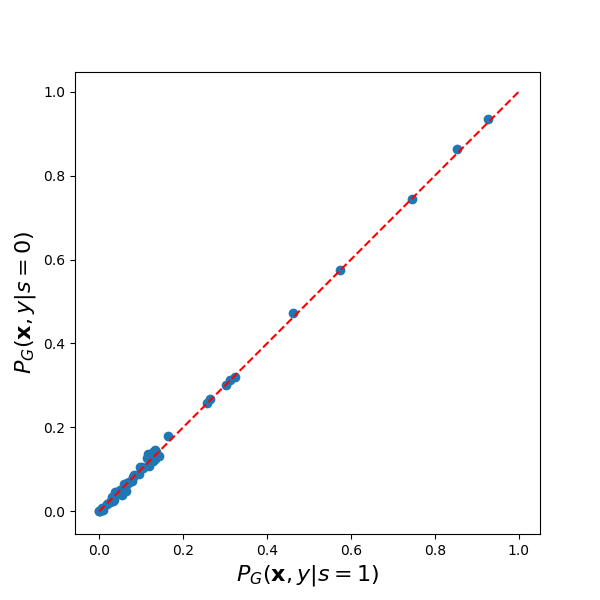}
		\caption{SYN3-NFGANII}
		\label{fig:f3}
	\end{subfigure}
	\begin{subfigure}{0.19\textwidth}
		\includegraphics[width=\linewidth]{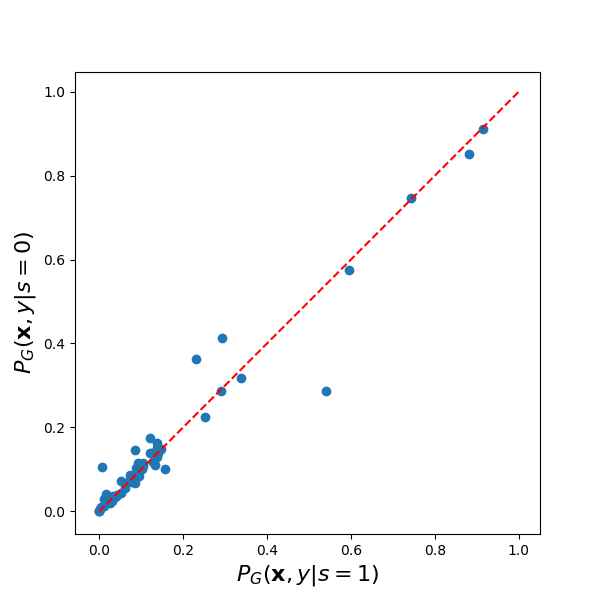}
		\caption{SYN4-FairGAN}
		\label{fig:f4}
	\end{subfigure}
	\caption{\textbf{Dimension-wise conditional probability distributions} ($P(\mathbf{x},y|s=1)$ vs. $P(\mathbf{x},y|s=0)$). Each dot represents one attribute. The x-axis represents the conditional probability given $s=1$. The y-axis represents the conditional probability given $s=0$. The diagonal line indicates the ideal fairness, where data have identical conditional probability distributions given $s$.}
	\label{fig:dwcp}
\end{figure*}

\section{Experiments}
We evaluate the performance of FairGAN on fair data generation and fair classification. 

\subsection{Experimental Setup}

{\bf \noindent Baselines.}
To evaluate the effectiveness of FairGAN, we compare the performance of FairGAN with the regular GAN model and two Na\"iveFairGAN models.
\textbf{GAN} aims to generate the synthetic samples that have the same distribution as the real data, i.e., $P_G(\mathbf{x},y,s)=P_{\text{data}}(\mathbf{x},y,s)$. The regular GAN model cannot achieve fair data generation. We adopt GAN as a baseline to evaluate the utility of data generation.
Na\"iveFairGAN models include \textbf{Na\"iveFairGAN-I} and \textbf{Na\"iveFairGAN-II}.

In this paper, we don't compare with the pre-process methods, because the classical methods like Massaging cannot remove disparate treatment and disparate impact \cite{Kamiran2012Data}. Although the certifying framework proposed algorithms to remove disparate impact, they only work on numerical attributes \cite{Feldman2015Certifying}.

{\bf \noindent Datasets.}
We evaluate FairGAN and baselines on the UCI Adult income dataset which contains 48,842 instances \cite{Dheeru2017Uci}. The decision indicates whether the income is higher than \$50k per year, and the protected attribute is gender. Each instance in the dataset consists of 14 attributes. We convert each attribute to a one-hot vector and combine all of them to a feature vector with 57 dimensions. 

In our experiments, besides adopting the original Adult dataset, we also generate four types of synthetic data, \textbf{SYN1-GAN} that is generated by a regular GAN model, \textbf{SYN2-NFGANI} that is generated by Na\"iveFairGAN-I, \textbf{SYN3-NFGANII} that is generated by Na\"iveFairGAN-II, and \textbf{SYN4-FairGAN} that is generated by FairGAN with $\lambda=1$. For each type of synthetic data, we generate five datasets to evaluate the data fairness and classification fairness. We then report the mean and stand deviation of evaluation results. The sizes of the synthetic datasets are same as the real dataset.

{\bf \noindent Implementation Details.}
We first pretrain the autoencoder for 200 epochs. Both the encoder $Enc$ and the decoder $Dec$ have one hidden layer with the dimension size as 128. The generator $G$ is a feedforward neural network with two hidden layers, each having 128 dimensions. The discriminator $D$ is also a feedforward network with two hidden layers where the first layer has 256 dimensions and the second layer has 128 dimensions. Na\"iveFairGAN-II and FairGAN are first trained without the fair constraint for 2,000 epochs ($D_1$ and $G_{Dec}$) and then trained with the fair constraint for another 2,000 epochs ($D_1$, $D_2$ and $G_{Dec}$). The regular GAN and Na\"iveFairGAN-I are trained for 2,000 epochs. We adopt Adam \cite{Kingma2015Adam} with the learning rate as 0.001 for stochastic optimization. 


\subsection{Fair Data Generation}

We evaluate FairGAN on data generation from two perspectives, fairness and utility. Fairness is to check whether FairGAN can generate fair data, while the utility is to check whether FairGAN can learn the distribution of real data precisely.

\begin{table}[h]
	\centering
	\caption{Risk differences of real and synthetic datasets}
	\label{tbl:cjpd}
	\resizebox{\columnwidth}{!}{%
	\begin{tabular}{|l|c|c|c|c|c|}
	\hline
	& Real Data & SYN1-GAN & SYN2-NFGANI & SYN3-NFGANII & SYN4-FairGAN \\ \hline
	$disk(\mathcal{D}) $                                                                         & 0.1989    & 0.1798$\pm$0.0026                                              & 0.0025$\pm$0.0007                                                    & 0.0062$\pm$0.0037                                                    & 0.0411$\pm$0.0295                                                 \\ \hline
	\end{tabular}
	}
\end{table}

\begin{figure}[h!]
	\centering 
	\begin{subfigure}{0.15\textwidth}
		\includegraphics[width=\linewidth]{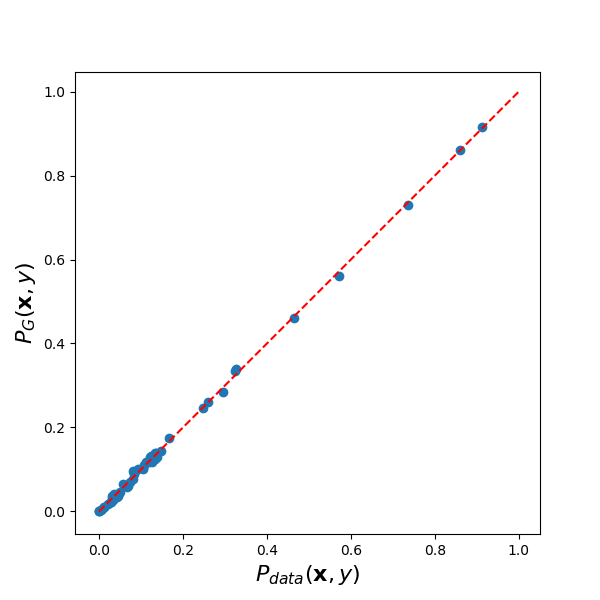}
		\caption{SYN1-GAN:\\ $P_{\text{data}}(\mathbf{x},y)$ vs. $P_G(\mathbf{x},y)$ \\~}
		\label{fig:gan_join}
	\end{subfigure}\hfil 
	\begin{subfigure}{0.15\textwidth}
		\includegraphics[width=\linewidth]{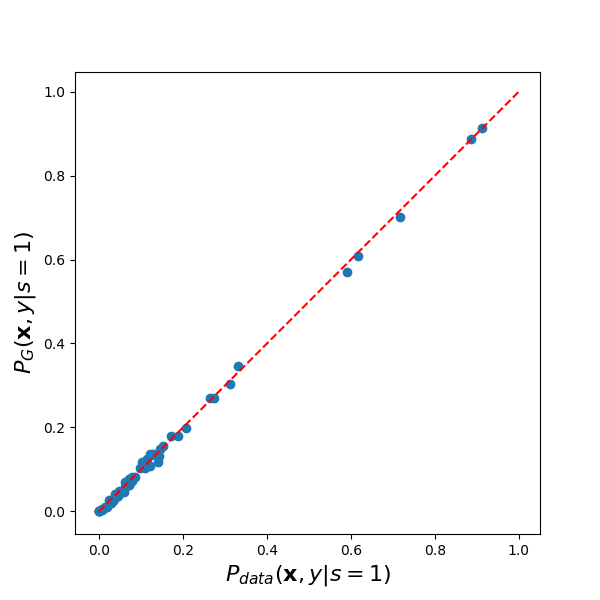}
		\caption{SYN1-GAN: \\$P_{\text{data}}(\mathbf{x},y|s=1)$ vs.\\ $P_G(\mathbf{x},y|s=1)$}
		\label{fig:gan_s1}
	\end{subfigure}\hfil 
	\begin{subfigure}{0.15\textwidth}
		\includegraphics[width=\linewidth]{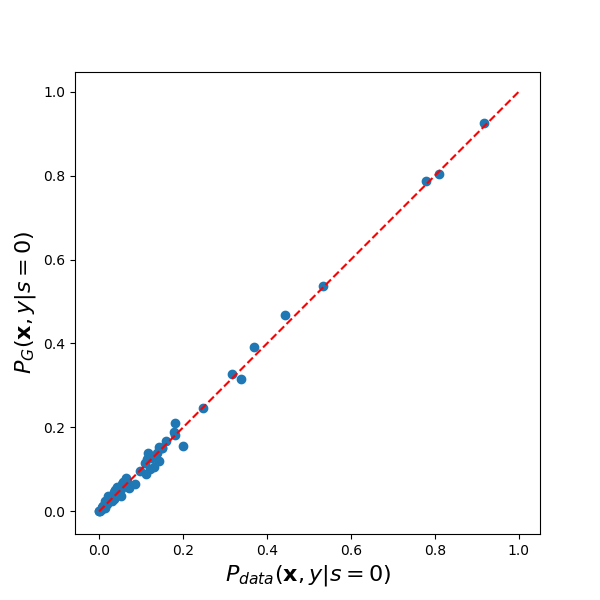}
		\caption{SYN1-GAN:\\ $P_{\text{data}}(\mathbf{x},y|s=0)$ vs.\\ $P_G(\mathbf{x},y|s=0)$}
		\label{fig:gan_s0}
	\end{subfigure}
	
	\medskip
	\begin{subfigure}{0.15\textwidth}
		\includegraphics[width=\linewidth]{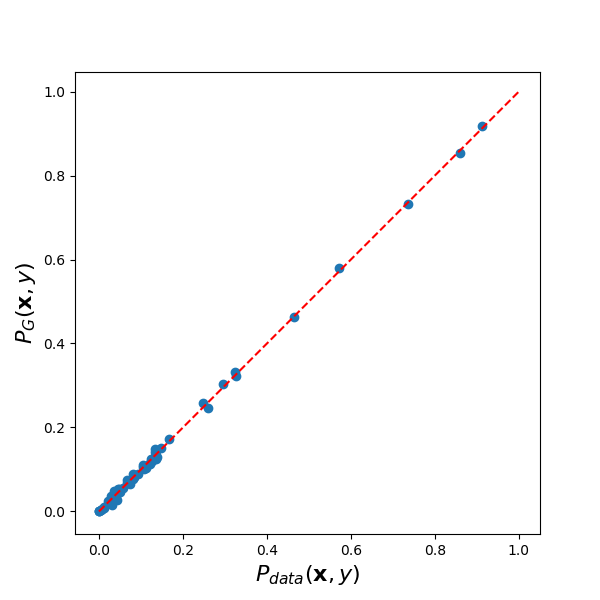}
		\caption{SYN2-NFGANI:\\ $P_{\text{data}}(\mathbf{x},y)$ vs. $P_G(\mathbf{x},y)$ \\~}
		\label{fig:ngan_join}
	\end{subfigure}\hfil 
	\begin{subfigure}{0.15\textwidth}
		\includegraphics[width=\linewidth]{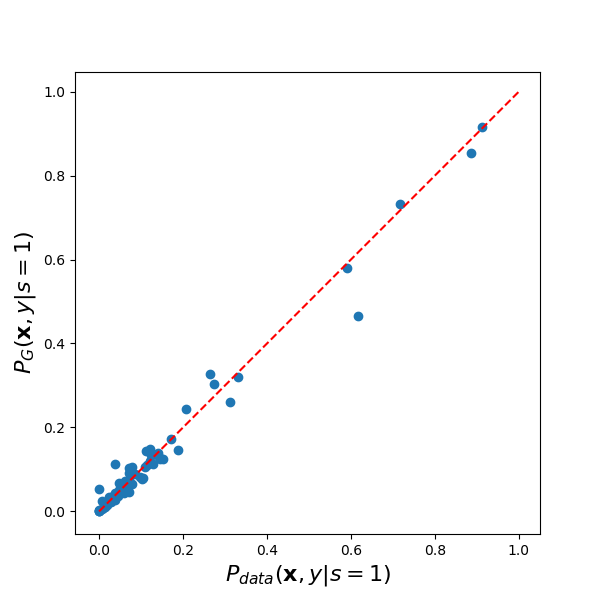}
		\caption{SYN2-NFGANI: \\$P_{\text{data}}(\mathbf{x},y|s=1)$ vs.\\  $P_G(\mathbf{x},y|s=1)$}
		\label{fig:ngan_s1}
	\end{subfigure}\hfil 
	\begin{subfigure}{0.15\textwidth}
		\includegraphics[width=\linewidth]{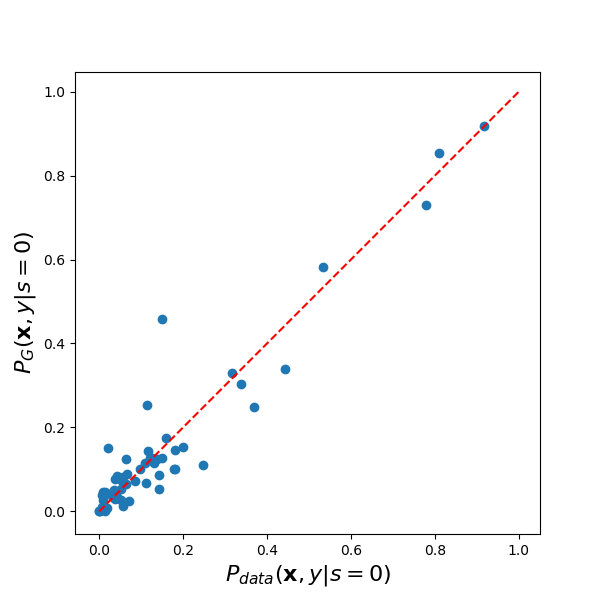}
		\caption{SYN2-NFGANI:\\ $P_{\text{data}}(\mathbf{x},y|s=0)$ vs.\\  $P_G(\mathbf{x},y|s=0)$}
		\label{fig:ngan_s0}
	\end{subfigure}

	\medskip
	\begin{subfigure}{0.15\textwidth}
		\includegraphics[width=\linewidth]{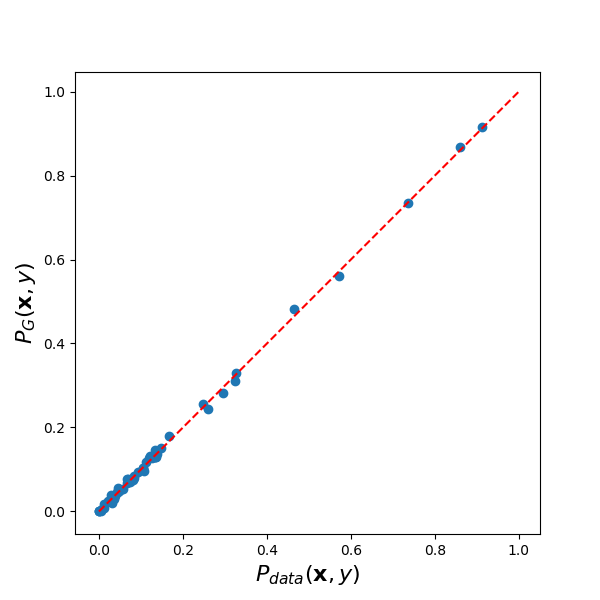}
		\caption{SYN3-NFGANII:\\ $P_{\text{data}}(\mathbf{x},y)$ vs. $P_G(\mathbf{x},y)$ \\~}
		\label{fig:nganii_join}
	\end{subfigure}\hfil 
	\begin{subfigure}{0.15\textwidth}
		\includegraphics[width=\linewidth]{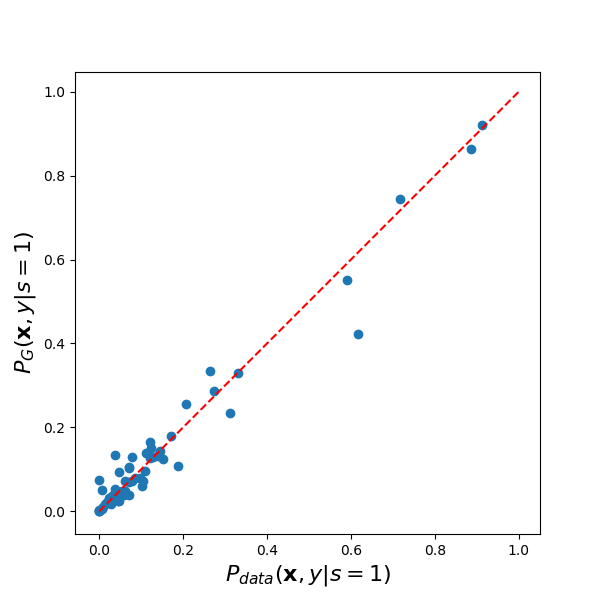}
		\caption{SYN3-NFGANII: \\$P_{\text{data}}(\mathbf{x},y|s=1)$ vs.\\  $P_G(\mathbf{x},y|s=1)$}
		\label{fig:nganii_s1}
	\end{subfigure}\hfil 
	\begin{subfigure}{0.15\textwidth}
		\includegraphics[width=\linewidth]{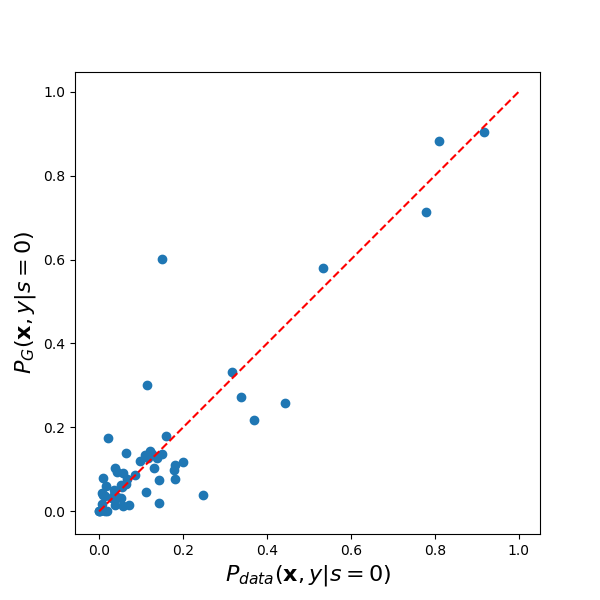}
		\caption{SYN3-NFGANII:\\ $P_{\text{data}}(\mathbf{x},y|s=0)$ vs.\\  $P_G(\mathbf{x},y|s=0)$}
		\label{fig:nganii_s0}
	\end{subfigure}
	
	\medskip
	\begin{subfigure}{0.15\textwidth}
		\includegraphics[width=\linewidth]{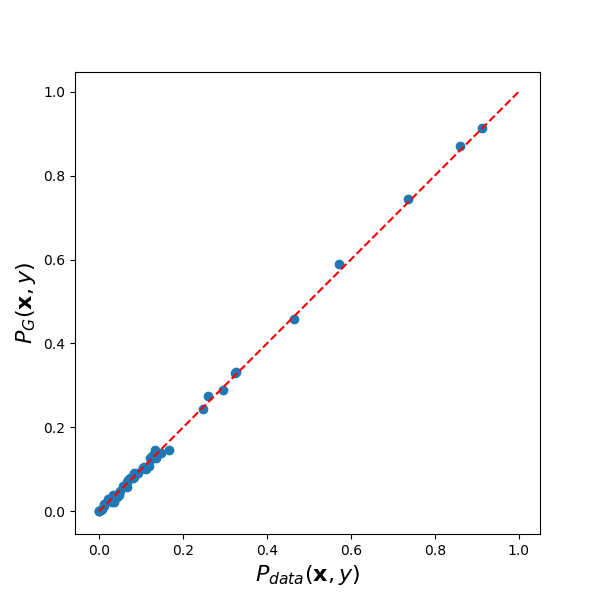}
		\caption{SYN4-FairGAN:\\ $P_{\text{data}}(\mathbf{x},y)$ vs. $P_G(\mathbf{x},y)$ \\~}
		\label{fig:fairgan_join}
	\end{subfigure}\hfil 
	\begin{subfigure}{0.15\textwidth}
		\includegraphics[width=\linewidth]{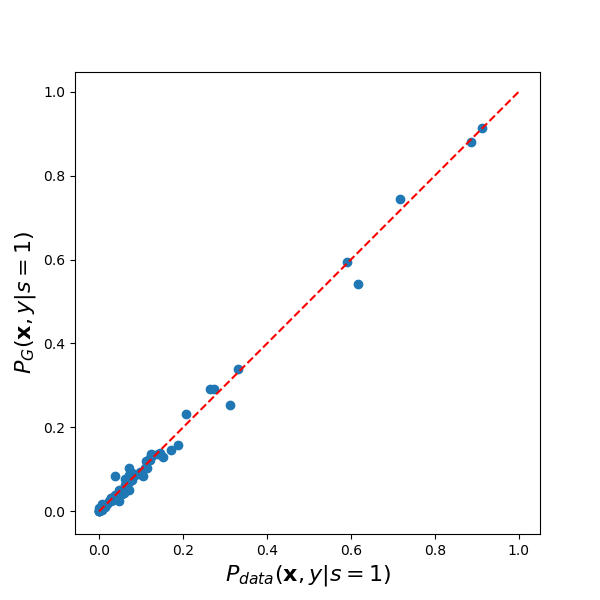}
		\caption{SYN4-FairGAN:\\ $P_{\text{data}}(\mathbf{x},y|s=1)$ vs.\\ $ P_G(\mathbf{x},y|s=1)$}
		\label{fig:fairgan_s1}
	\end{subfigure}\hfil 
	\begin{subfigure}{0.15\textwidth}
		\includegraphics[width=\linewidth]{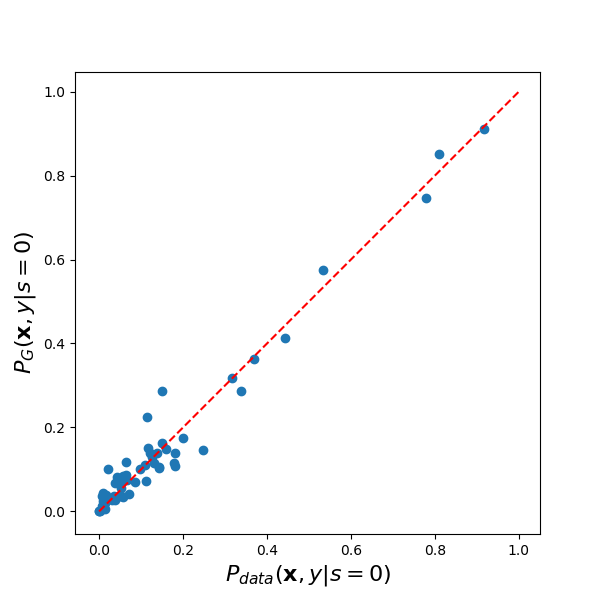}
		\caption{SYN4-FairGAN:\\ $P_{\text{data}}(\mathbf{x},y|s=0)$ vs.\\ $P_G(\mathbf{x},y|s=0)$}
		\label{fig:fairgan_s0}
	\end{subfigure}
	\caption{Dimension-wise probability distributions. Each dot represents one attribute. The x-axis represents the Bernoulli success probability for the real dataset. The y-axis represents the probability for the synthetic dataset generated by each model. The diagonal line indicates the ideal case, where the real and synthetic data show identical quality.}
	\label{fig:dwp}
\end{figure}

{\bf \noindent Fairness.}
We adopt the \textit{risk difference in a labeled dataset} ($disc(\mathcal{D}) = P(y=1|s=1)-P(y=1|s=0)$) as the metric to compare the performance of different GAN models on fair data generation. Table \ref{tbl:cjpd} shows the risk differences in the real and synthetic datasets. The risk difference in the Adult dataset is 0.1989, which indicates discrimination against female. The SYN-GAN, which is trained to be close to the real dataset, has the similar risk difference to the real dataset. On the contrary, SYN2-NFGANI, SYN3-NFGANII, and SYN4-FairGAN have lower risk differences than the real dataset. In particular, both SYN2-NFGANI and SYN3-NFGANII have extremely small risk differences. This is because the protected attribute of SYN2-NFGANI and SYN3-NFGANII is independently assigned, i.e., $\hat{y} \independent \hat{s}$. Hence, the synthetic datasets from SYN2-NFGANI and SYN3-NFGANII are free from disparate treatment. FairGAN prevents the disparate treatment by generating revised $\hat{y}$ to make $\hat{y} \independent \hat{s}$.
The risk difference of SYN4-FairGAN is 0.0411, which shows the effectiveness of FairGAN on fair data generation.  


In Figure \ref{fig:dwcp}, we compare the dimension-wise conditional probability distributions between $P(\mathbf{x},y|s=1)$ and $P(\mathbf{x},y|s=0)$. Each dot indicates one attribute. The diagonal line indicates the ideal fairness, where the conditional probability distributions of each attribute given $s = 1$ and $s = 0$ are identical. We can observe that the dimension-wise distributions of datasets with lower risk differences are closer to the diagonal line. For example, dimension-wise conditional probabilities of real dataset and SYN1-GAN are spread around the diagonal line (shown in Figures \ref{fig:f0} and \ref{fig:f1}), while conditional probabilities of SYN2-NFGANI and SYN3-NFGANII with the lowest risk difference are just on the diagonal line (shown in Figure \ref{fig:f2}). SYN4-FairGAN also achieves reasonable risk difference, so the attribute dots are close to the diagonal line. Overall, the synthetic datasets from SYN2-NFGANI, SYN3-NFGANII, and SYN4-FairGAN can prevent the disparate treatment.

We further evaluate the $\epsilon$-fairness (disparate impact) by calculating the balanced error rates (BERs) in the real data and SYN4-FairGAN. Because the protected attribute in SYN2-NFGANI and SYN3-NFGANII are randomly assigned, the real $s$ given  $\hat{\mathbf{x}}$ is unknown. The BERs in SYN2-NFGANI and SYN3-NFGANII cannot be calculated. The BER in the real dataset is 0.1538, which means a classifier can predict $s$  given $\mathbf{x}$ with high accuracy. Hence, there is disparate impact in the real dataset. On the contrary, the BER in SYN4-FairGAN is 0.3862$\pm$0.0036, which indicates using the generated $\hat{\mathbf{x}}$ in SYN4-FairGAN to predict the real $s$ has much higher error rate. The disparate impact in SYN4-FairGAN is small. It shows the effectiveness of FairGAN on removal of the disparate impact in terms of the real $s$. Note that we adopt a linear SVM as a classifier to predict $s$.

{\bf \noindent Utility.}
We then evaluate the data utility of synthetic datasets. We adopt the dimension-wise probability to check whether the generated data have the similar distribution to the real data on each dimension. Figure \ref{fig:dwp} compares dimension-wise probability distributions of different GAN models in both joint probability $P(\mathbf{x}, y)$ and conditional probability $P(\mathbf{x},y|s)$. From Figures \ref{fig:gan_join}, \ref{fig:ngan_join}, \ref{fig:nganii_join} and \ref{fig:fairgan_join}, we can observe that the four synthetic datasets generated by different GAN models have similar $P(\mathbf{x}, y)$ to the real dataset. Meanwhile, $P_G(\mathbf{x},y|s=1)$ and $P_G(\mathbf{x},y|s=0)$ on SYN1-GAN perfectly match the real dataset (shown in Figures \ref{fig:gan_s1} and \ref{fig:gan_s0}), which indicates the effectiveness of the regular GAN model on data generation. We can also observe that SYN4-FariGAN better preserves $P(\mathbf{x},y|s)$ than SYN2-NFGANI and SYN3-NFGANII by comparing the Figures \ref{fig:fairgan_s1} and \ref{fig:fairgan_s0} with Figures \ref{fig:ngan_s1}, \ref{fig:ngan_s0}, \ref{fig:nganii_s1} and \ref{fig:nganii_s0}. This is because neither Na\"iveFairGAN-I nor Na\"iveFairGAN-II ensures the generated samples have the same conditional probability distribution given $s$ as the real data.

In Table \ref{tbl:jpd}, we further evaluate the closeness between each synthetic dataset and the real dataset by calculating the Euclidean distance of joint and conditional probabilities ($P(\mathbf{x},y)$, $P(\mathbf{x},y,s)$, and $P(\mathbf{x},y|s)$). 
The Euclidean distance is calculated between the estimated probability vectors (probability mass function) on the sample space from the synthetic dataset and the real dataset. A smaller distance indicates better closeness between the real data and the synthetic data. As expected, SYN1-GAN has the smallest distance to the real dataset for joint and conditional probabilities. For synthetic datasets generated by FairGAN and Na\"iveFairGAN models, SYN2-NFGANI has the smallest distance in terms of $||P_{\text{data}}(\mathbf{x},y)-P_G(\mathbf{x},y)||_2$ since its objective is $P_G (\mathbf{x},y)=P_{\text{data}}(\mathbf{x},y)$, while SYN4-FairGAN has the smallest distance in terms of conditional probability $||P_{\text{data}}(\mathbf{x},y|s)-P_G(\mathbf{x},y|s)||_2$ and joint probability $||P_{\text{data}}(\mathbf{x},y,s)-P_G(\mathbf{x},y,s)||_2$ since only FairGAN aims to ensure $P_G (\mathbf{x},y,s)=P_{\text{data}}(\mathbf{x},y,s)$. Overall, without considering the protected attribute, all the synthetic datasets from FairGAN and Na\"iveFairGAN models are close to the real dataset. When considering the protected attribute, FairGAN has better performance than Na\"iveFairGAN models. Therefore, after  removing disparate impact, FairGAN  still achieves good data utility. 

\begin{table}[h]
	\centering
	\caption{Euclidean distances of joint and conditional probabilities between synthetic datasets and real dataset}
	\label{tbl:jpd}
	\resizebox{\columnwidth}{!}{%
\begin{tabular}{|l|c|c|c|c|}
\hline
Euclidean Distance                                                                                     & SYN1-GAN  & SYN2-NFGANI & SYN3-NFGANII & SYN4-FairGAN \\ \hline
\begin{tabular}[c]{@{}l@{}}$||P_{\text{data}}(\mathbf{x},y)$\\ $-P_G(\mathbf{x},y)||_2$\end{tabular}     & 0.0231$\pm$0.0003                                              & 0.0226$\pm$0.0003                                                    & 0.0227$\pm$0.0003                                                    & 0.0233$\pm$0.0004                                                  \\ \hline
\begin{tabular}[c]{@{}l@{}}$||P_{\text{data}}(\mathbf{x},y|s=1)$\\ $-P_G(\mathbf{x},y|s=1)||_2$\end{tabular} & 0.0108$\pm$0.0002                                              & 0.0118$\pm$0.0003                                                    & 0.0118$\pm$0.0005                                                    & 0.0111$\pm$0.0004                                                  \\ \hline
\begin{tabular}[c]{@{}l@{}}$||P_{\text{data}}(\mathbf{x},y|s=0)$\\ $-P_G(\mathbf{x},y|s=0)||_2$\end{tabular} & 0.0166$\pm$0.0002                                              & 0.0194$\pm$0.0003                                                    & 0.0200$\pm$0.0007                                                    & 0.0176$\pm$0.0005  \\ \hline
\begin{tabular}[c]{@{}l@{}}$||P_{\text{data}}(\mathbf{x},y,s)$\\ $-P_G(\mathbf{x},y,s)||_2$\end{tabular} & 0.0198$\pm$0.0002                                              & 0.0227$\pm$0.0003                                                    & 0.0232$\pm$0.0006                                                    & 0.0208$\pm$0.0005                                                  \\ \hline
\end{tabular}
}
\end{table}

\begin{table*}[h]
\centering
\caption{Risk differences in classifiers and classification accuracies on various training and testing settings}
\label{tb:rd-accuracy}
\resizebox{\textwidth}{!}{%
\begin{tabular}{|l|c|c|c|c|c|c|c|c|c|c|}
\hline
\multirow{2}{*}{}                                                          & \multirow{2}{*}{Classifier} & REAL2REAL & \multicolumn{4}{c|}{SYN2SYN}                                                                                                                                                                                                          & \multicolumn{4}{c|}{SYN2REAL}                                                                                                                                                                                                         \\ \cline{3-11} 
                                                                           &                             &           & \begin{tabular}[c]{@{}c@{}}SYN1-\\ GAN\end{tabular} & \begin{tabular}[c]{@{}c@{}}SYN2-\\ NFGANI\end{tabular} & \begin{tabular}[c]{@{}c@{}}SYN3-\\ NFGANII\end{tabular} & \begin{tabular}[c]{@{}c@{}}SYN4-\\ FairGAN\end{tabular} & \begin{tabular}[c]{@{}c@{}}SYN1-\\ GAN\end{tabular} & \begin{tabular}[c]{@{}c@{}}SYN2-\\ NFGANI\end{tabular} & \begin{tabular}[c]{@{}c@{}}SYN3-\\ NFGANII\end{tabular} & \begin{tabular}[c]{@{}c@{}}SYN4-\\ FairGAN\end{tabular} \\ \hline
\multirow{3}{*}{\begin{tabular}[c]{@{}l@{}}Risk\\ Difference\end{tabular}} & SVM (Linear)                & 0.1784    & 0.1341$\pm$0.0023                                              & 0.0018$\pm$0.0021                                                    & 0.0073$\pm$0.0039                                                    & 0.0371$\pm$0.0189                                                  & 0.1712$\pm$0.0062                                              & 0.1580$\pm$0.0076                                                    & 0.1579$\pm$0.0079                                                   & 0.0461$\pm$0.0424                                                  \\ \cline{2-11} 
                                                                           & SVM (RBF)                   & 0.1788    & 0.1292$\pm$0.0049                                              & 0.0018$\pm$0.0025                                                    & 0.0074$\pm$0.0028                                                    & 0.0354$\pm$0.0206                                                  & 0.1623$\pm$0.0050                                              & 0.1602$\pm$0.0053                                                    & 0.1603$\pm$0.0087                                                    & 0.0526$\pm$0.0353                                                  \\ \cline{2-11} 
                                                                           & Decision Tree               & 0.1547    & 0.1396$\pm$0.0089                                              & 0.0015$\pm$0.0035                                                    & 0.0115$\pm$0.0061                                                    & 0.0535$\pm$0.0209                                                  & 0.1640$\pm$0.0077                                              & 0.1506$\pm$0.0070                                                    & 0.1588$\pm$0.0264                                                    & 0.0754$\pm$0.0641                                                  \\  \hline \hline
\multirow{3}{*}{Accuracy}                                                  & SVM (Linear)                & 0.8469    & 0.8281$\pm$0.0103                                              & 0.8162$\pm$0.0133                                                    & 0.0.8226$\pm$0.0126                                                    & 0.8247$\pm$0.0115                                                  & 0.8363$\pm$0.0108                                              & 0.8340$\pm$0.0091                                                    & 0.8356$\pm$0.0018                                                    & 0.8217$\pm$0.0093                                                  \\ \cline{2-11} 
                                                                           & SVM (RBF)                   & 0.8433    & 0.8278$\pm$0.0099                                             &  0.8160$\pm$0.0100   & 0.8215$\pm$0.0130                                                    & 0.8233$\pm$0.0103                                                 & 0.8342$\pm$0.0036                                               & 0.8337$\pm$0.0060                                                   & 0.8349$\pm$0.0012                                                    & 0.8178$\pm$0.0128                                                  \\ \cline{2-11} 
                                                                           & Decision Tree               & 0.8240    & 0.8091$\pm$0.0059                                              & 0.7926$\pm$0.0083                                                    & 0.8055$\pm$0.0102                                                    & 0.8077$\pm$0.0144                                                  & 0.8190$\pm$0.0051                                              & 0.8199$\pm$0.0041                                                    & 0.8158$\pm$0.0069                                                    & 0.8044$\pm$0.0140                                                  \\ \hline
\end{tabular}
}
\end{table*}

\subsection{Fair Classification}
In this subsection, we adopt the real and synthetic datasets to train several classifiers and check whether the classifiers can achieve fairness. We evaluate the classifiers with three settings: 1) the classifiers are trained and tested on the real dataset, called \textbf{REAL2REAL}; 2) the classifiers are trained and tested on the synthetic datasets, called \textbf{SYN2SYN}; 3) the classifiers are trained on the synthetic datasets and tested on the real dataset, called \textbf{SYN2REAL}. The ratio of the training set to testing set in these three settings is 1:1. We emphasize that only SYN2REAL is meaningful in practice as the classifiers are trained from the generated data and are adopted for decision making on the real data. 

We adopt the following classifiers to evaluate the fair classification: 1) \textbf{SVM (linear)} which is a linear support vector machine with $C=1$; 2) \textbf{SVM (RBF)} which is a support vector machine with the radial basis kernel function; 3) \textbf{Decision Tree} with maximum tree depth as 5; Note that we do not adopt the protected attribute and only use the unprotected attributes to train classifiers, which ensures there is no disparate treatment in classifiers. 

{\bf \noindent Fairness.}
We adopt the risk difference in a classifier ($disc(\eta) = P(\eta(\mathbf{x})=1|s=1)-P(\eta(\mathbf{x})=1|s=0)$) to evaluate the performance of classifier on fair prediction. Table \ref{tb:rd-accuracy} shows the risk differences in classifiers on various training and testing settings. We can observe that when the classifiers are trained and tested on real datasets (i.e., REAL2REAL), the risk differences in classifiers are high. It indicates that if there is disparate impact in the training dataset, the classifiers also incur discrimination for prediction. Since SYN1-GAN is close to the real dataset, classifiers trained on SYN1-GAN also have discrimination in both SYN2SYN and SYN2REAL settings. 

Although SYN2-NFGANI and SYN3-NFGANII have similar distributions as the real dataset on unprotected attributes and decision, i.e., $P_G(\mathbf{x}, y)=P_{\text{data}}(\mathbf{x}, y)$, classifiers which are trained and tested in SYN2SYN settings achieve low risk differences. This is because the values of the protected attribute in SYN2-NFGANI or SYN3-NFGANII are independently generated or shuffled. Since both $\hat{\mathbf{x}}$ and  $\hat{y}$ have no correlations with the generated $\hat{s}$ in SYN2-NFGANI and SYN3-NFGANII, the statistical parity in classifiers can be achieved when trained and tested on synthetic datasets. 

However, when classifiers are trained on SYN2-NFGANI and SYN3-NFGANII and tested on the real dataset (i.e., SYN2REAL), the classifiers still have significant discrimination against the protected group. Because the unprotected attributes of SYN2-NFGANI and SYN3-NFGANII are close to the real dataset, the correlations between the generated $\hat{\mathbf{x}}$ and the real $s$ are still preserved. The disparate impact in terms of the real $s$ on SYN2-NFGANI and SYN3-NFGANII is not removed. When classifiers are tested on the real dataset where the correlations between $\mathbf{x}$ and $s$ are preserved, the classification results indicate discrimination. On the contrary, when the classifiers are trained on SYN4-FairGAN and tested on the real dataset, we can observe that the risk differences in classifiers are small. Since the FairGAN prevents the discrimination by generating  $\hat{\mathbf{x}}$ that don't have correlations with the real $s$, i.e. free from the disparate impact, the classifier trained on SYN4-FairGAN can achieve fair classification on the real dataset. It demonstrates the advantage of FairGAN over the Na\"iveFairGAN models on fair classification.

{\bf \noindent Classification accuracy.}
Table \ref{tb:rd-accuracy} further shows the classification accuracies of different classifiers on various training and testing settings. We can observe that the accuracies of classifiers on the SYN2REAL setting are close to the results on the REAL2REAL setting. It indicates synthetic datasets generated by different GAN models are similar to the real dataset, showing the good data generation utility of GAN models.  Meanwhile, accuracies of classifiers which are trained on SYN4-FairGAN and tested on real dataset are only slightly lower than those trained on SYN1-GAN, which means the FairGAN model can achieve a good balance between utility and fairness. The small utility loss is caused by modifying unprotected attributes to remove disparate impact in terms of the real $s$.

\subsection{Parameter Sensitivity}
\begin{figure}[h!]
	\centering
	\begin{subfigure}{.23\textwidth}
		\includegraphics[width=\linewidth]{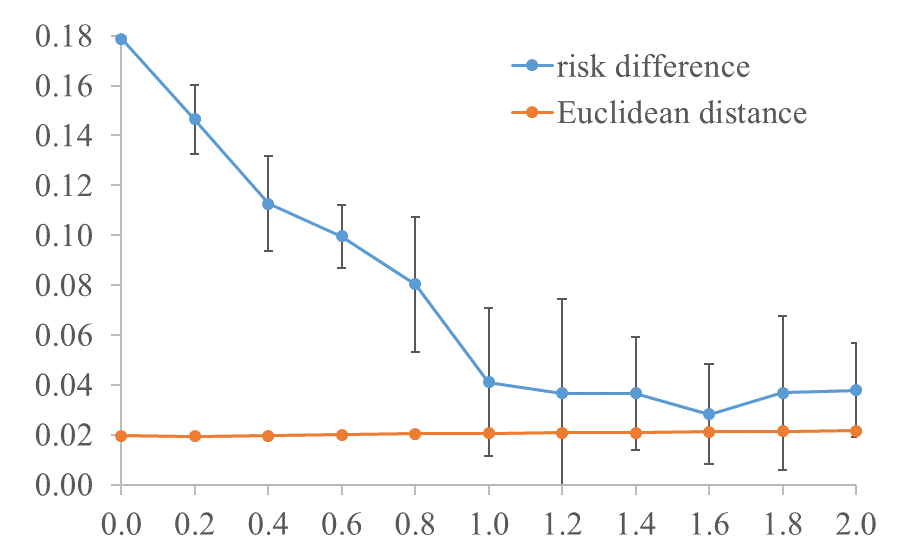}
		\caption{Utility and fairness in synthetic datasets from FairGAN with various $\lambda$. \\ ~}
		\label{fig:lambda_data}
	\end{subfigure}\hfil
	\begin{subfigure}{.23\textwidth}
		\includegraphics[width=\linewidth]{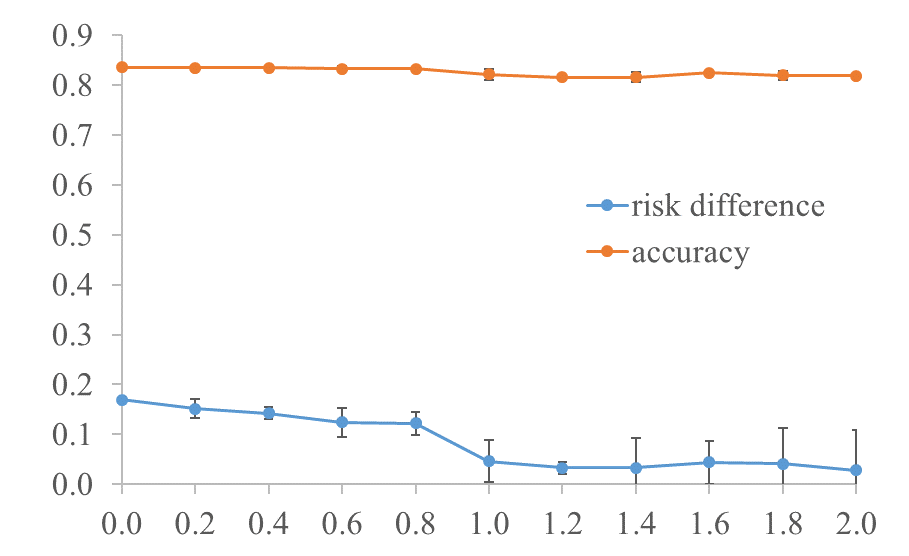}
		\caption{Accuracy and fairness in a linear SVM which is trained on synthetic datasets from FairGAN with various $\lambda$ and tested on real dataset.}
		\label{fig:lambda_classfication}
	\end{subfigure}
	\caption{The sensitivity analysis of FairGAN with various $\lambda$}
\end{figure}

We evaluate how the $\lambda$ in FairGAN affects the synthetic datasets for fair data generation and fair classification. For fair data generation, we evaluate risk differences of the generated datasets and the Euclidean distances of joint probabilities $||P_{\text{data}}(\mathbf{x},y,s)-P_{G}(\mathbf{x},y,s)||_2$ between real and synthetic datasets. From Figure \ref{fig:lambda_data}, we can observe that the risk differences of the generated datasets decrease significantly when $\lambda$ increases. Meanwhile, the Euclidean distances of joint probabilities $||P_{\text{data}}(\mathbf{x},y,s)-P_{G}(\mathbf{x},y,s)||_2$ keep steady with slightly increases while $\lambda$ changes from 0 to 2. Meanwhile, the standard deviations of Euclidean distances with various $\lambda$ are smaller than $10^{-3}$. Overall, with the increase of $\lambda$ from 0 to 2, the discrimination in the synthetic datasets becomes smaller while data generation utility keeps steady. 

For fair classification, we train a linear SVM on different synthetic datasets generated by FairGAN with various $\lambda$ and evaluate on the real dataset. Figure \ref{fig:lambda_classfication} shows how the accuracies and risk differences vary with different $\lambda$ values. We can observe that the risk difference in SVM when predicting on the real dataset decreases as $\lambda$ increases. Meanwhile, the prediction accuracy keeps relatively steady with a slightly decrease. The standard deviations of accuracies with various $\lambda$ are smaller than $10^{-2}$. Overall, it indicates that increasing $\lambda$ can prevent the classification discrimination on the real dataset while achieving good classification utility.

\section{Conclusions and Future Work}
In this paper, we have developed FairGAN to generate fair data, which is free from disparate treatment and disparate impact in terms of  the real protected attribute, while retaining high data utility. As a result, classifiers trained on the generated fair data are not subject to discrimination when making decision of the real data. FairGAN consists of one generator and two discriminators. In particular, the generator generates fake samples conditioned on the protected attribute. One discriminator is trained to identify whether samples are real or fake, while the other discriminator is trained to distinguish whether the generated samples are from the protected group or unprotected group. The generator can generate fair data with high utility by playing the adversarial games with these two discriminators. The experimental results showed the effectiveness of FairGAN.  In future, we plan to improve FairGAN to make classifiers trained on synthetic datasets achieve equalized odds or equal opportunity on the real dataset besides statistical parity. We emphasize this is the first work to study the use GAN for generating fair data, which is different from pre-process approaches \cite{Kamiran2009Classifying,Kamiran2012Data,Zhang2017Achieving,Feldman2015Certifying} widely studied in fairness aware learning. In our future work, we will conduct comprehensive comparisons both theoretically and empirically in terms of fairness-utility tradeoff. 

\begin{acks}
This work was supported in part by NSF 1646654 and 1564250.
\end{acks}

\bibliographystyle{ACM-Reference-Format}
\bibliography{Remote}

\section*{Appendix}

{\bf \noindent Theoretical analysis of NaiveFairGAN-II}

For $G$ fixed, the optimal discriminators $D_1$ and $D_2$ by the value function $V(G_{Dec}, D_1, D_2)$ are

\resizebox{0.98\linewidth}{!}{
$
	D_1^*(\mathbf{x},y) = \frac{P_{\text{data}}(\mathbf{x},y|s=1)+P_{\text{data}}(\mathbf{x},y|s=0)} {P_{\text{data}}(\mathbf{x},y|s=1)+P_{\text{data}}(\mathbf{x},y|s=0)+P_{G}(\mathbf{x},y|s=1)+P_{G}(\mathbf{x},y|s=0)},
$
}

\[
	D_2^*(\mathbf{x},y) = \frac{P_{G}(\mathbf{x},y|s=1)}{P_{G}(\mathbf{x},y|s=1)+P_{G}(\mathbf{x},y|s=0)}
\]
 The proof procedure is a straightforward extension of the proof in Proposition \ref{prop:fairgan}.


Given $D_1^*$ and $D_2^*$, we can reformulate the minimax game of Na\"iveFairGAN-II as:

\resizebox{.98\linewidth}{!}{
\begin{minipage}{\linewidth}
\begin{align*}
&~C'(G)  = \max_{D_1,D_2} V(G_{Dec}, D_1, D_2) \\
      & = \mathbb{E}_{(\mathbf{x},y) \sim P_{\text{data}}(\mathbf{x},y|s=1)}[\log \frac{P_{\text{data}}(\mathbf{x},y|s=1)+P_{\text{data}}(\mathbf{x},y|s=0)} {P_{\text{data}}(\mathbf{x},y|s=1)+P_{\text{data}}(\mathbf{x},y|s=0)+P_{G}(\mathbf{x},y|s=1)+P_{G}(\mathbf{x},y|s=0)}] \\
	   & + \mathbb{E}_{(\mathbf{x},y) \sim P_{\text{data}}(\mathbf{x},y|s=0)}[\log \frac{P_{\text{data}}(\mathbf{x},y|s=1)+P_{\text{data}}(\mathbf{x},y|s=0)} {P_{\text{data}}(\mathbf{x},y|s=1)+P_{\text{data}}(\mathbf{x},y|s=0)+P_{G}(\mathbf{x},\hat{y}|s=1)+P_{G}(\mathbf{x},y|s=0)}] \\
	   & + \mathbb{E}_{(\mathbf{x},y) \sim P_{G}(\mathbf{x},y|s=1)}[\log \frac{P_{G}(\mathbf{x},y|s=1)+P_{G}(\mathbf{x},y|s=0)} {P_{\text{data}}(\mathbf{x},y|s=1)+P_{\text{data}}(\mathbf{x},y|s=0)+P_{G}(\mathbf{x},y|s=1)+P_{G}(\mathbf{x},y|s=0)}] \\
	   & + \mathbb{E}_{(\mathbf{x},y) \sim P_{G}(\mathbf{x},y|s=0)}[\log \frac{P_{G}(\mathbf{x},y|s=1)+P_{G}(\mathbf{x},y|s=0)} {P_{\text{data}}(\mathbf{x},y|s=1)+P_{\text{data}}(\mathbf{x},y|s=0)+P_{G}(\mathbf{x},y|s=1)+P_{G}(\mathbf{x},y|s=0)}] \\
	   & + \mathbb{E}_{(\mathbf{x},y) \sim P_{G}(\mathbf{x},y|s=1)} [\log \frac{P_{G}(\mathbf{x},y|s=1)} {P_{G}(\mathbf{x},y|s=1)+P_{G}(\mathbf{x},y|s=0)}] \\
	   & + \mathbb{E}_{(\mathbf{x},y) \sim P_{G}(\mathbf{x},y|s=0)} [\log \frac{P_{G}(\mathbf{x},y|s=0)} {P_{G}(\mathbf{x},y|s=1)+P_{G}(\mathbf{x},y|s=0)}] \\
\end{align*}
\end{minipage}
}

Follow the proof of GAN \cite{Goodfellow2014Generative}, $C'(G)$ can be rewritten as

\resizebox{0.99\linewidth}{!}{
\begin{minipage}{\linewidth}
\begin{align}
\label{eq:nfgan_cg}
	& ~C'(G) = -4 \log4 + 2 \cdot JSD(P_{G}(\mathbf{x},y|s=1)||P_{G}(\mathbf{x},y|s=0)) \notag \\ 
	& + 2JSD \Big(P_{\text{data}}(\mathbf{x},y|s=1)+P_{\text{data}}(\mathbf{x},y|s=0)||P_{G}(\mathbf{x},y|s=1)+P_{G}(\mathbf{x},y|s=0)\Big).   
\end{align}
\end{minipage}
}

Since Jensen-Shannon divergence between two distributions is non-negative and zero only when they are equal, the only solution of $C'(G)$ reaching the global minimum is that the two JSD terms in Equation \ref{eq:nfgan_cg} are both zeros, i.e., $P_{\text{data}}(\mathbf{x},y|s=1)+P_{\text{data}}(\mathbf{x},y|s=0) = P_{G}(\mathbf{x},y|s=1)+P_{G}(\mathbf{x},y|s=0)$ and $P_{G}(\mathbf{x},y|s=1) = P_{G}(\mathbf{x},y|s=0)$ simultaneously. Hence, the global minimum of Na\"iveFairGAN-II is $P_{G}(\mathbf{x},y|s=1) = P_{G}(\mathbf{x},y|s=0) = \big(P_{\text{data}}(\mathbf{x},y|s=1)+P_{\text{data}}(\mathbf{x},y|s=0)\big)/2$. At that point, $V(G_{Dec}, D_1, D_2)$ achieves the value $-4 \log4$.


\end{document}